\documentclass[conference]{IEEEtran}
\IEEEoverridecommandlockouts
\usepackage{cite}
\usepackage{amsmath,amssymb,amsfonts}
\usepackage{algorithmic}
\usepackage{graphicx}
\usepackage{textcomp}
\usepackage{xcolor}
\def\BibTeX{{\rm B\kern-.05em{\sc i\kern-.025em b}\kern-.08em
    T\kern-.1667em\lower.7ex\hbox{E}\kern-.125emX}}

\usepackage{booktabs} % For formal tables
\usepackage{graphicx}
\usepackage{balance}  % for  \balance command ON LAST PAGE  (only there!)
\usepackage[ruled,vlined,linesnumbered]{algorithm2e}
\usepackage{amsthm}
\usepackage{adjustbox}
\usepackage{xspace}
\usepackage{hyperref}
\usepackage{url}
\usepackage{autobreak}
\usepackage{color}
\usepackage{fancyvrb}
\usepackage{tablefootnote} % For table footnote
\usepackage[center]{subfigure} % for \subfloat
\usepackage[mathscr]{euscript} % for euler script.
\usepackage{makecell}
\usepackage{footnote} % For footnotes in tables (\savenotes)
\usepackage{mdwlist}
\usepackage{multirow}
\usepackage{enumitem}

\newtheorem{definition}{Definition}
\newtheorem{lemma}{Lemma}
\newtheorem{theorem}{Theorem}

\SetKwComment{Comment}{$\triangleright$\ }{}

\newcommand{\method}{\textsc{DPar2}\xspace}

\newcommand{\als}{PARAFAC2-ALS\xspace}
\newcommand{\rdals}{RD-ALS\xspace}
\newcommand{\spartan}{SPARTan\xspace}
\newcommand{\footnoteref}[1]{\textsuperscript{\ref{#1}}}

\newcommand{\hide}[1]{}

\newcommand{\mat}[1]{\mathbf{#1}}
\newcommand{\matt}[1]{\mathbf{#1}^{\text{T}}}

\newcommand*{\QEDB}{\hfill\ensuremath{\Box}}%
\newcommand*\concat{\mathbin{\|}}

\DeclareMathOperator*{\argmin}{argmin}

\newcommand{\T}[1]{\boldsymbol{\mathscr{#1}}}

\newcommand{\subfloat}{\subfigure}

\newcommand\new[1]{\textcolor[RGB]{20,93,160}{#1}}

\begin{document}

\title{DPar2: Fast and Scalable PARAFAC2 Decomposition for Irregular Dense Tensors
}

\author{\IEEEauthorblockN{Jun-Gi Jang}
\IEEEauthorblockA{\textit{Computer Science and Engineering} \\
\textit{Seoul National University}\\
Seoul, Republic of Korea \\
elnino4@snu.ac.kr}
\and
\IEEEauthorblockN{U Kang}
\IEEEauthorblockA{\textit{Computer Science and Engineering} \\
\textit{Seoul National University}\\
Seoul, Republic of Korea \\
ukang@snu.ac.kr}
%\and
%\and
%\and
%\and
}

\maketitle

\begin{abstract}
  Given an irregular dense tensor, how can we efficiently analyze it?
An irregular tensor is a collection of matrices whose columns have the same size and rows have different sizes from each other.
PARAFAC2 decomposition is a fundamental tool to deal with an irregular tensor in applications including phenotype discovery and trend analysis.
Although several PARAFAC2 decomposition methods exist,
their efficiency is limited for irregular dense tensors due to the expensive computations involved with the tensor.
%However, recent works have focused on sparse irregular tensors even though there are various irregular dense tensors such as stock, sound, and music data.
%Since they exploit the sparsity structure of a given irregular tensor, there is no benefit when the given tensor is dense.

In this paper, we propose \method, a fast and scalable PARAFAC2 decomposition method for irregular dense tensors.
\method achieves high efficiency by effectively compressing each slice matrix of a given irregular tensor, careful reordering of computations with the compression results, and exploiting the irregularity of the tensor.
Extensive experiments show that \method is up to $6.0\times$ faster than competitors on real-world irregular tensors while achieving comparable accuracy.
In addition, \method is scalable with respect to the tensor size and target rank. 
\end{abstract}


\section{Introduction}
\label{sec:intro}
\textit{How can we efficiently analyze an irregular dense tensor?}
Many real-world multi-dimensional arrays are represented as irregular dense tensors;
%the lengths of a mode vary, and almost all of the entries are observable.
an irregular tensor is a collection of matrices with different row lengths.
For example, stock data can be represented as an irregular dense tensor;
the listing period is different for each stock \textbf{(irregularity)}, and
%\textbf{(temporality)} each stock has a temporal history of features (e.g., closing price, volume, and moving average), and
almost all of the entries of the tensor are observable during the listing period \textbf{(high density)}.
The irregular tensor of stock data is the collection of the stock matrices whose row and column dimension corresponds to time and features (e.g., the opening price, the closing price, the trade volume, etc.), respectively.
In addition to stock data, many real-world data including music song data and sound data are also represented as irregular dense tensors.
Each song can be represented as a slice matrix (e.g., time-by-frequency matrix) whose rows correspond to the time dimension. Then, the collection of songs is represented as an irregular tensor consisting of slice matrices of songs each of whose time length is different.
Sound data are represented similarly.

%\begin{table}[t!]
%	\small
%	\caption{Summary of \method and competitors with respect to the performance properties.
%	\method with the same rank is faster and more scalable than other methods.
%	\method with a slightly higher rank, which uses slightly larger factor matrices than other competitors, satisfies all the properties.
%				}
%	\centering
%		\resizebox{\columnwidth}{!}{%
%	\begin{tabular}{l | c c c c c }
%		\toprule
%		\textbf{Method}  & \textbf{Speed} & \textbf{Space} & \textbf{Scale} & \textbf{Fitness} \\
%		\midrule
%		\als	&  &  &  & \checkmark  \\
%		\spartan	&  & &  & \checkmark \\
%		\rdals	&  & \checkmark &  &  \\		
%		\midrule
%%		 \textbf{\method w/ the same rank}	 & \checkmark & \checkmark & \checkmark \\
%		\textbf{\method}	 & \checkmark & \checkmark & \checkmark  &\checkmark \\
%			 &  &  &   & \blue{(with a higher rank)} \\
%		\bottomrule
%	\end{tabular}}
%	\label{tab:salesman}
%%	\vspace{-3mm}
%\end{table}

Tensor decomposition has attracted much attention from the data mining community to analyze tensors~\cite{lin2009metafac,spiegel2011link,jang2021fast,oh2019high,kwon2021slicenstitch,DBLP:conf/cikm/AhnKK21,Ahn2021,DBLP:conf/cikm/AhnSK20,10.1371/journal.pone.0217316,vldbj/Park2019}.
Specifically, PARAFAC2 decomposition has been widely used for modeling irregular tensors in various applications including phenotype discovery~\cite{PerrosPWVSTS17,AfsharPPSHS18}, trend analysis~\cite{helwig2017estimating}, and fault detection~\cite{wise2001application}.
% as shown in Fig.~\ref{fig:example}
However, existing PARAFAC2 decomposition methods are not fast and scalable enough for irregular dense tensors.
Perros et al.~\cite{PerrosPWVSTS17} improve the efficiency for handling irregular sparse tensors, by exploiting the sparsity patterns of a given irregular tensor.
Many recent works~\cite{AfsharPPSHS18,Ren00H20,afshar2020taste,yin2020logpar} adopt their idea to handle irregular sparse tensors.
However, they are not applicable to irregular \emph{dense} tensors that have no sparsity pattern.
Although Cheng and Haardt~\cite{ChengH19} improve the efficiency of PARAFAC2 decomposition by preprocessing a given tensor, there is plenty of room for improvement in terms of computational costs.
Moreover, there remains a need for fully employing multicore parallelism.
The main challenge to successfully design a fast and scalable PARAFAC2 decomposition method is how to minimize the computational costs involved with an irregular dense tensor and the intermediate data generated in updating factor matrices.

In this paper, we propose \method (\underline{D}ense \underline{PAR}AFAC\underline{2} decomposition), a fast and scalable PARAFAC2 decomposition method for irregular dense tensors.
Based on the characteristics of real-world data, \method compresses each slice matrix of a given irregular tensor using randomized Singular Value Decomposition (SVD).
The small compressed results and our careful ordering of computations considerably reduce the computational costs and the intermediate data.
In addition, \method maximizes multi-core parallelism by considering the difference in sizes between slices.
With these ideas, \method achieves higher efficiency and scalability than existing PARAFAC2 decomposition methods on irregular dense tensors.
Extensive experiments show that \method outperforms the existing methods in terms of speed, space, and scalability, while achieving a comparable fitness, where
the fitness indicates how a method approximates a given data well (see Section~\ref{subsec:setting}).
%Table~\ref{tab:salesman} summarizes a comparison between \method and existing methods in terms of speed, space, scalability, and fitness, where
%the fitness indicates how a method approximates a given data well (see Section~\ref{subsec:setting}).
%\method satisfies the aspects
%with a higher rank, which uses slightly larger factor matrices than others, satisfies all aspects.

\begin{figure*}
	\centering
%	\vspace{-4mm}
	 \subfloat{\includegraphics[width=0.9\textwidth]{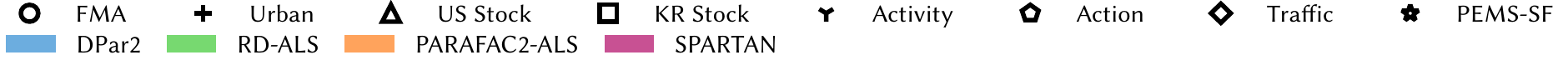}} \\
	\vspace{-3mm}	
	 \setcounter{subfigure}{0}
	 \subfloat[Trade-off]{\includegraphics[width=0.24\textwidth]{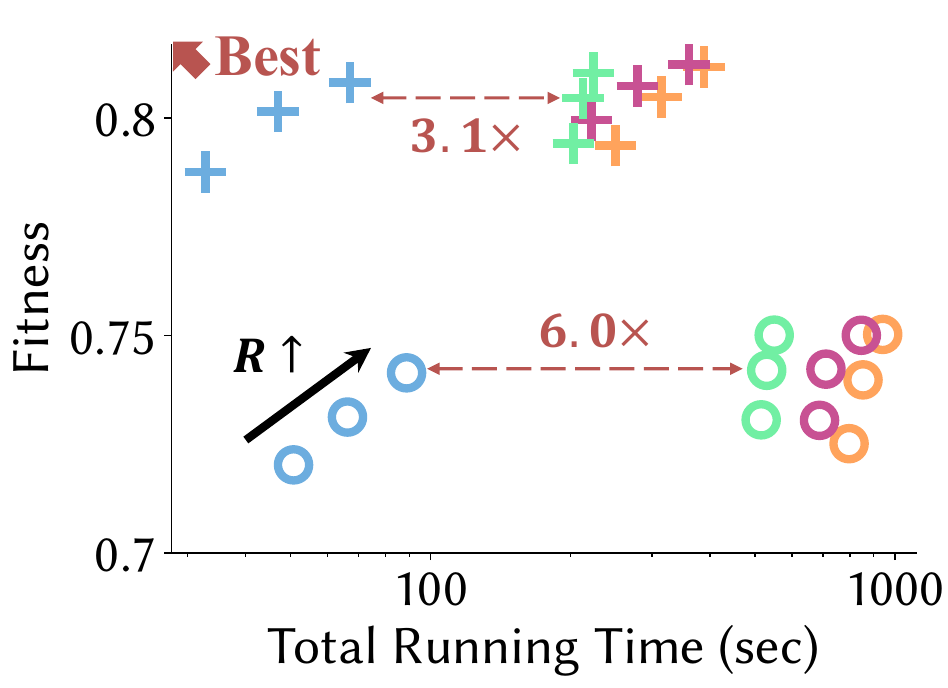}\label{fig:trade_off1}}
	 \subfloat[Trade-off]{\includegraphics[width=0.24\textwidth]{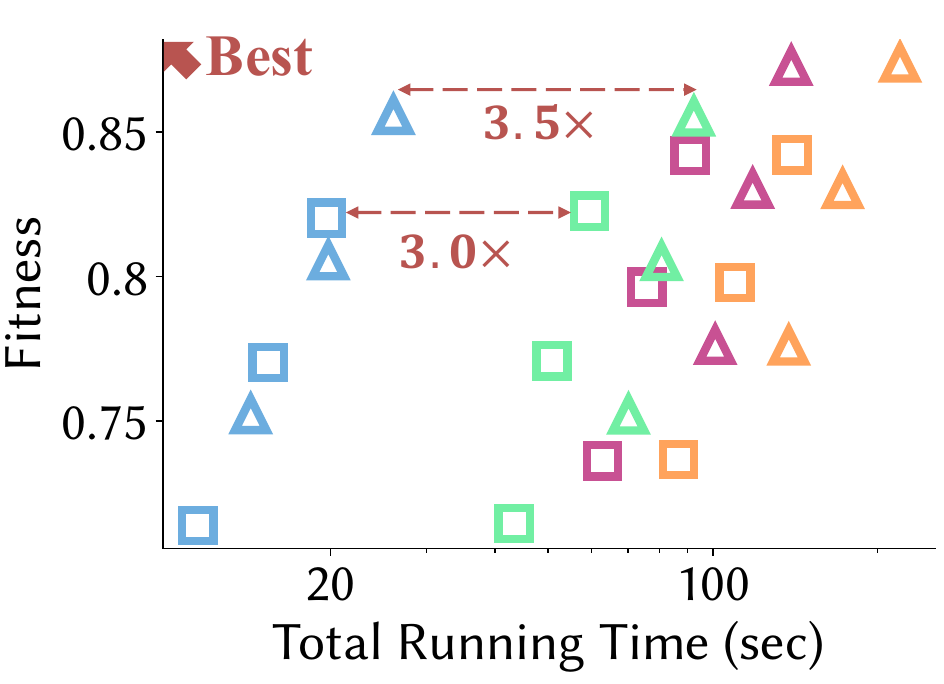}\label{fig:trade_off2}}
	 \subfloat[Trade-off]{\includegraphics[width=0.24\textwidth]{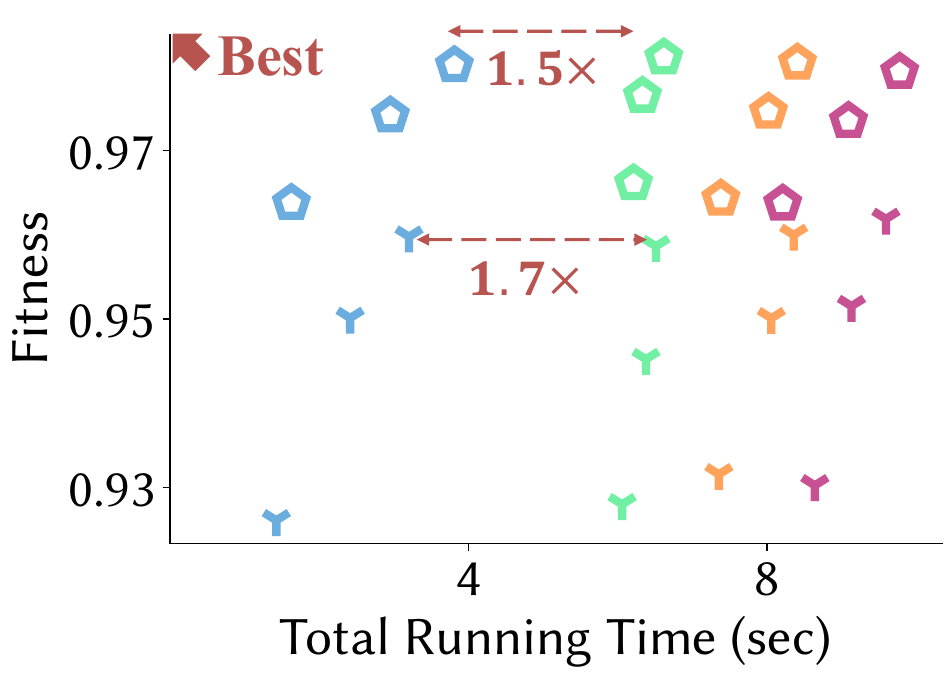}\label{fig:trade_off3}}
	 \subfloat[Trade-off]{\includegraphics[width=0.24\textwidth]{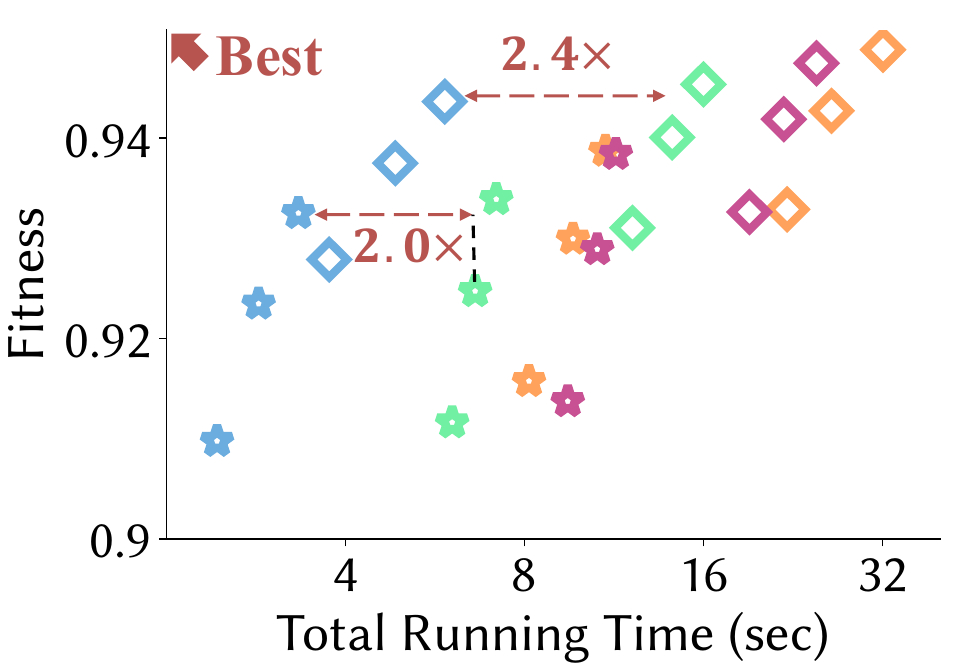}\label{fig:trade_off4}}	 	
%	 \qquad
%	 \subfloat[Iteration time]{\includegraphics[width=0.24\textwidth]{FIG/PERF/ITER_TIME.pdf}\label{fig:iter_time}}	
%	 \subfloat[Preprocessing time]{\includegraphics[width=0.24\textwidth]{FIG/PERF/PREP_TIME.pdf}\label{fig:prep_time}}	
%	 \qquad
%	\includegraphics[width=0.3\textwidth]{FIG/PERF/MEMORY.pdf}
%	 \subfloat[The size of preprocessed data]{\includegraphics[width=0.32\textwidth]{FIG/PERF/MEMORY.pdf}\label{fig:preprocessed_size}}
%	 \subfloat[Scalability for dimensionality]{\includegraphics[width=0.24\textwidth]{FIG/PERF/SCALABILITY_DIMEN.pdf}\label{fig:scala_dimen}}
%	 \subfloat[Scalability for rank]{\includegraphics[width=0.24\textwidth]{FIG/PERF/SCALABILITY_RANK.pdf}\label{fig:scala_rank}}	
%	 \subfloat[KR Stock]{\includegraphics[width=0.24\textwidth]{FIG/PERF/PERF_STOCK_KR.pdf}\label{fig:perf_stock_kr}}	 	 	
	 	 \\
	 	\vspace{-1mm}
	\caption{\textbf{[Best viewed in color]}
	Measurement of the running time and fitness on real-world datasets for three target ranks $R$: $10$, $15$, and $20$.
%For the maximum iteration (=$32$), 
\method provides the best trade-off between speed and fitness.
	\method is up to $6.0\times$ faster than the competitors while having a comparable fitness.
%	(c) At each iteration, \method outperforms competitors by up to $11.4\times$.
%	(d) \method efficiently preprocesses a given irregular dense tensor by up to $10.1\times$ compared to \rdals.
%	(c) The size of preprocessed data. \method generates up to $189\times$ smaller preprocessed data than input tensor which is used for competitors.
%Thanks to the small preprocessed data, \method has low memory requirements at each iteration.
%	\green{See Section~\ref{subsec:exp:querytime} for details.}
	}
	\label{fig:performance}
\end{figure*}

The contributions of this paper are as follows.
\begin{itemize*}
	\item {\textbf{Algorithm.} We propose \method, a fast and scalable PARAFAC2 decomposition method for decomposing irregular dense tensors.}
	\item {\textbf{Analysis.} We provide analysis for the time and the space complexities of our proposed method \method.}
%	\item Also, we analyze the error bound of \method.}
	\item \textbf{Experiment.} \method achieves up to $6.0\times$ faster running time than previous PARAFAC2 decomposition methods based on ALS while achieving a similar fitness (see Fig.~\ref{fig:performance}).
	\item \textbf{Discovery.} With \method, we find that the Korean stock market and the US stock market have different correlations (see Fig.~\ref{fig:V_sim}) between features (e.g., prices and technical indicators). We also find similar stocks (see Table~\ref{tab:similar_discovery}) on the US stock market during a specific event (e.g., COVID-19).
\end{itemize*}

%\begin{figure}
%	\centering
%%	\vspace{-4mm}
%	\subfloat[]{\includegraphics[width=0.49\textwidth]{FIG/PROBLEM.pdf}
%	\caption{
%	Example of an irregular tensor for stock data.
%	}
%	\label{fig:problem}
%\end{figure}

In the rest of this paper,
we describe the preliminaries in Section~\ref{sec:prelim}, propose our method \method in Section~\ref{sec:method},
present experimental results in Section~\ref{sec:experiment},
%demonstrate the case study in Section~\ref{sec:casestudy},
discuss related works in Section~\ref{sec:related},
and conclude in Section \ref{sec:conclusion}.
The code and datasets are available at \textbf{\url{https://datalab.snu.ac.kr/dpar2}}. 

\section{Preliminaries}
\label{sec:prelim}

In this section, we describe tensor notations, tensor operations, Singular Value Decomposition (SVD), and PARAFAC2 decomposition.
We use the symbols listed in Table~\ref{tab:notation}.
%\begin{figure}
%	\centering
%%	\vspace{-4mm}
%%	 \subfloat{\includegraphics[width=0.3\textwidth]{FIG/PERF/PERF_LEGEND.pdf}} \\
%%	\vspace{-3mm}	
%%	 \setcounter{subfigure}{0}
%	 \includegraphics[width=0.4\textwidth]{FIG/PERF/TRADE_OFF1.pdf}
%%	 \subfloat[Trade-off]{\includegraphics[width=0.235\textwidth]{FIG/PERF/TRADE_OFF1.pdf}\label{fig:trade_off}}	
%%	 \subfloat[Urban]{\includegraphics[width=0.24\textwidth]{FIG/PERF/PERF_URBAN.pdf}\label{fig:perf_urban}}
%%	 \subfloat[US Stock]{\includegraphics[width=0.24\textwidth]{FIG/PERF/PERF_STOCK_US.pdf}\label{fig:perf_stock_us}}	 	
%%	 \subfloat[KR Stock]{\includegraphics[width=0.24\textwidth]{FIG/PERF/PERF_STOCK_KR.pdf}\label{fig:perf_stock_kr}}	 	 	
%	 	 \\
%%	 	\vspace{-1mm}
%	\caption{ [Best viewed in color]
%	For each iteration, we measure the running time and fitness.
%	Note that the left most point of \method includes the preprocessing time and the running time of the first iteration.
%	For the maximum iteration (=$15$), \method is up to $55.1\times$ faster than \als while sacrificing little accuracy ($<= 6\%$ point).
%%	\green{See Section~\ref{subsec:exp:querytime} for details.}
%	}
%	\label{fig:tradeoff}
%\end{figure}

\begin{table} [t]
	\centering
	\caption{Symbol description.}
	\label{tab:notation}
	\footnotesize
	\resizebox{\columnwidth}{!}{%
		\begin{tabular}{cl}
			\toprule
			\textbf{Symbol} & \textbf{Description} \\
			\midrule
			$\{ \mat{X}_k \}^{K}_{k=1}$ & irregular tensor of slices $\mat{X}_{k}$ for $k=1,...,K$ \\
			$\mat{X}_{k}$ & slice matrix ($\in I_k \times J$)  \\
			$\mat{X}(i,:)$ & $i$-th row of a matrix $\mat{X}$ \\
			$\mat{X}(:,j)$ & $j$-th column of a matrix $\mat{X}$ \\			
			$\mat{X}(i,j)$ & $(i,j)$-th element of a matrix $\mat{X}$ \\						
			$\mat{X}_{(n)}$ & mode-$n$ matricization of a tensor $\T{X}$ \\
			$\mat{Q}_k$, $\mat{S}_{k}$ & factor matrices of the $k$th slice \\
			$\mat{H}$, $\mat{V}$ & factor matrices of an irregular tensor \\	
			$\mat{A}_k$, $\mat{B}_k$, $\mat{C}_{k}$ & SVD results of the $k$th slice \\
			$\mat{D}$, $\mat{E}$, $\mat{F}$ & SVD results of the second stage \\
			$\mat{F}^{(k)}$ & $k$th vertical block matrix $(\in \mathbb{R}^{R\times R})$ of $\mat{F} (\in \mathbb{R}^{KR \times R})$ \\
			$\mat{Z}_{k}$, $\mat{\Sigma_{k}}$, $\mat{P}_{k}$ & SVD results of $\mat{F}^{(k)}\mat{E}\mat{D}^T\mat{V}\mat{S}_{k}\mat{H}^{T}$ \\
			$R$ & target rank \\	
			$\otimes$ & Kronecker product \\
			$\odot$ & Khatri-Rao product \\			
			$*$ & element-wise product \\
			$\concat$ & horizontal concatenation \\		
			$vec(\cdot)$ & vectorization of a matrix \\
			\bottomrule
		\end{tabular}
	}
\end{table}

\subsection{Tensor Notation and Operation}
%The length of each mode is called \lq dimensionality\rq\:and denoted by $I_1,\cdots,I_N $.
We use boldface lowercases (e.g. $\mat{x}$) and boldface capitals (e.g. $\mat{X}$) for vectors and matrices, respectively.
In this paper, indices start at $1$.
An irregular tensor is a 3-order tensor $\T{X}$ whose $k$-frontal slice $\T{X}(:,:,k)$ is $\mat{X}_{k} \in \mathbb{R}^{I_k \times J}$.
We denote irregular tensors by $\{ \mat{X}_k \}^{K}_{k=1}$ instead of $\T{X}$ where $K$ is the number of $k$-frontal slices of the tensor.
An example is described in Fig.~\ref{fig:example}.
We refer the reader to \cite{KoldaB09} for the definitions of tensor operations including Frobenius norm, matricization, Kronecker product, and Khatri-Rao product.

\subsection{Singular Value Decomposition (SVD)}
\label{subsec:SVD}
{Singular Value Decomposition (SVD) decomposes $\mat{A} \in \mathbb{R}^{I \times J}$ to $\mat{X} = \mat{U}\mat{\Sigma}\matt{V}$.
$\mat{U} \in \mathbb{R}^{I \times R}$ is the left singular vector matrix of $\mat{A}$;
$\mat{U}=\begin{bmatrix} \mathbf{u}_1 \cdots \mathbf{u}_r \end{bmatrix}$ is a column orthogonal matrix where $R$ is the rank of $\mat{A}$ and $\mathbf{u}_1$, $\cdots$, $\mathbf{u}_R$ are the eigenvectors of $\mat{A}\matt{A}$.
$\mat{\Sigma}$ is an $R \times R$ diagonal matrix whose diagonal entries are singular values.
The $i$-th singular value $\sigma_i$ is in $\mat{\Sigma}_{i,i}$ where $\sigma_1$ $\geq$ $\sigma_2$ $\geq$ $\cdots$ $\geq$ $\sigma_{R}$ $\geq$ $0$.
$\mat{V} \in \mathbb{R}^{J \times R}$ is the right singular vector matrix of $\mat{A}$; $\mat{V}=\begin{bmatrix} \mathbf{v}_1 \cdots \mathbf{v}_R \end{bmatrix}$ is a column orthogonal matrix where $\mathbf{v}_1$, $\cdots$, $\mathbf{v}_R$ are the eigenvectors of $\matt{A}\mat{A}$.
}
%Note that the singular vectors in $\mat{U}$ and $\mat{V}$ are used as hidden factors to analyze the data matrix $\mat{X}$.

\begin{algorithm} [t]
	\SetNoFillComment
	\caption{Randomized SVD~\protect\cite{HalkoMT11}}
	\label{alg:randomized_svd}
	\begin{algorithmic} [1]
%		\footnotesize
		\small
		\algsetup{linenosize=\small}

		\renewcommand{\algorithmicrequire}{\textbf{Input:}}
		\renewcommand{\algorithmicensure}{\textbf{Output:}}
		    \REQUIRE $\mat{A} \in \mathbb{R}^{I \times J}$ %for $k = 1,...,K$
		    \ENSURE $\mat{U} \in \mathbb{R}^{I \times R}$, $\mat{S} \in \mathbb{R}^{R \times R}$, and $\mat{V} \in \mathbb{R}^{J \times R}$.
		\renewcommand{\algorithmicrequire}{\textbf{Parameters:}}
		\REQUIRE target rank $R$, and an exponent $q$ \\
		\STATE generate a Gaussian test matrix $\mat{\Omega} \in \mathbb{R}^{J \times (R+s)}$
		\STATE construct $\mat{Y} \leftarrow (\mat{A}\mat{A}^T)^q\mat{A}\mat{\Omega}$
		\STATE $\mat{Q}\mat{R} \leftarrow \mat{Y}$ using QR factorization
		\STATE construct $\mat{B} \leftarrow \mat{Q}^T\mat{A}$
		\STATE $\tilde{\mat{U}}\mat{\Sigma}\mat{V}^T \leftarrow \mat{B}$ using truncated SVD at rank $R$
		\RETURN $\mat{U} = \mat{Q}\tilde{\mat{U}}$, $\mat{\Sigma}$, and $\mat{V}$
	\end{algorithmic}
\end{algorithm}

\textbf{Randomized SVD}.
Many works~\cite{woolfe2008fast,HalkoMT11,clarkson2017low} have introduced efficient SVD methods to decompose a matrix $\mat{A} \in \mathbb{R}^{I \times J}$ by applying randomized algorithms.
We introduce a popular randomized SVD in Algorithm~\ref{alg:randomized_svd}.
Randomized SVD finds a column orthogonal matrix $\mat{Q} \in \mathbb{R}^{I\times (R+s)}$ of $(\mat{A}\mat{A}^T)^q\mat{A}\mat{\Omega}$ using random matrix $\mat{\Omega}$,
constructs a smaller matrix $\mat{B} = \mat{Q}^T\mat{A}$ $(\in \mathbb{R}^{(R+s) \times J})$,
and finally obtains the SVD result $\mat{U}$ $(=\mat{Q}\tilde{\mat{U}})$, $\mat{\Sigma}$, $\mat{V}$ of $\mat{A}$ by computing SVD for $\mat{B}$, i.e., $\mat{B} \approx \tilde{\mat{U}}\mat{\Sigma}\mat{V}^T$.
Given a matrix $\mat{A}$, the time complexity of randomized SVD is $\T{O}(IJR)$ where $R$ is the target rank.
% describes how randomized SVD works in detail.

\subsection{PARAFAC2 decomposition}
%\begin{figure} [h]
%	\centering
%	\includegraphics [width=0.3\textwidth] {FIG/Tuckerdecomposition.pdf}
%%	\vspace{-6mm}
%	\caption{Tucker decomposition for 3-mode tensor.
%	$\mat{A}^{(1)}$, $\mat{A}^{(2)}$, and $\mat{A}^{(3)}$ are factor matrices, and $\T{G}$ is a core tensor.
%	}
%	\label{fig:tucker_decomposition}
%\end{figure}
PARAFAC2 decomposition proposed by Harshman~\cite{harshman1972parafac2} successfully deals with irregular tensors.
The definition of PARAFAC2 decomposition is as follows:
\begin{definition}[PARAFAC2 Decomposition]
	Given a target rank $R$ and a 3-order tensor $\{ \mat{X}_k \}^{K}_{k=1}$ whose $k$-frontal slice is $\mat{X}_{k} \in \mathbb{R}^{I_k \times J}$ for $k=1,...,K$, PARAFAC2 decomposition
	approximates each $k$-th frontal slice $\mat{X}_{k}$ by $\mat{U}_{k}\mat{S}_{k}\mat{V}^T$. $\mat{U}_{k}$ is a matrix of the size $I_k \times R$, $\mat{S}_{k}$ is a diagonal matrix of the size $R \times R$, and $\mat{V}$ is a matrix of the size $J \times R$ which are common for all the slices. \QEDB
\end{definition}

\begin{algorithm} [t]
	\SetNoFillComment
	\caption{PARAFAC2-ALS~\protect\cite{kiers1999parafac2}}
	\label{alg:als}
	\begin{algorithmic} [1]
%		\footnotesize
		\small
		\algsetup{linenosize=\small}

		\renewcommand{\algorithmicrequire}{\textbf{Input:}}
		\renewcommand{\algorithmicensure}{\textbf{Output:}}
		    \REQUIRE $\mat{X}_{k} \in \mathbb{R}^{I_{k} \times J}$ for $k = 1,...,K$
		    \ENSURE $\mat{U}_{k} \in \mathbb{R}^{I_{k} \times R}$, $\mat{S}_{k} \in \mathbb{R}^{R \times R}$ for $k=1,...,K$, and $\mat{V} \in \mathbb{R}^{J \times R}$.
		\renewcommand{\algorithmicrequire}{\textbf{Parameters:}}
		\REQUIRE target rank $R$ \\
		\STATE initialize matrices $\mat{H} \in \mathbb{R}^{R \times R}$, $\mat{V}$, and $\mat{S}_{k}$ for $k=1,...,K$
		\REPEAT \label{alg:line:iter_start}
		\FOR {$k=1,...,K$}
			\STATE compute $\mat{Z}'_{k}\mat{\Sigma'_{k}}\mat{P}_{k}^{'T} \leftarrow \mat{X}_{k}\mat{V}\mat{S}_{k}\mat{H}^{T}$ by performing truncated SVD at rank $R$ \label{alg:line:compute_X1} \\
			\STATE $\mat{Q}_{k} \leftarrow \mat{Z}'_{k}\mat{P}_{k}^{'T}$ \label{alg:line:compute_Qk} \\
		\ENDFOR
		\FOR {$k=1,...,K$}
			\STATE $\mat{Y}_{k} \leftarrow \mat{Q}_{k}^{T}\mat{X}_{k}$ \label{alg:line:compute_X2} \\
		\ENDFOR		
		\STATE construct a tensor $\T{Y} \in \mathbb{R}^{R \times J \times K}$ from slices $\mat{Y}_{k} \in \mathbb{R}^{R \times J}$ for $k=1,...,K$ \label{alg:line:construct_y} \\
		\tcc{running a single iteration of CP-ALS on $\T{Y}$ }
		\STATE $\mat{H} \leftarrow \mat{Y}_{(1)}(\mat{W}\odot \mat{V})(\mat{W}^{T}\mat{W} * 	\mat{V}^{T}\mat{V})^{\dagger}$ \label{alg:line:cp_start}\\
		\STATE $\mat{V} \leftarrow \mat{Y}_{(2)}(\mat{W}\odot \mat{H})(\mat{W}^{T}\mat{W} * 	\mat{H}^{T}\mat{H})^{\dagger}$\\
		\STATE $\mat{W} \leftarrow \mat{Y}_{(3)}(\mat{V}\odot \mat{H})(\mat{V}^{T}\mat{V} * 	\mat{H}^{T}\mat{H})^{\dagger}$ \label{alg:line:cp_end}\\
%		\STATE compute $\mat{H},\mat{V}$, and $\mat{W}$ by running a single iteration of CP-ALS on $\T{Y}$
		\FOR {$k=1,...,K$}
			\STATE $\mat{S}_{k} \leftarrow diag(\mat{W}(k,:))$
		\ENDFOR			\label{alg:line:update_s} \\
		\UNTIL{the maximum iteration is reached, or the error ceases to decrease;} \label{alg:line:iter_end}\\
		\FOR {$k=1,...,K$}
			\STATE $\mat{U}_{k} \leftarrow \mat{Q}_{k}\mat{H}$
		\ENDFOR	
%		\STATE \textbf{initialize:} factor matrices $\mat{A}^{(i)}$ ($i=1,...,N$)
% 		\REPEAT
% 			\FOR {$i=1,...,N$}
% 				\STATE{$\T{Y} \leftarrow \T{X}\times_1\mat{A}^{(1)T}\cdots\times_{i-1}\mat{A}^{(i-1)T}\times_{i+1}\mat{A}^{(i+1)T}\cdots\times_N\mat{A}^{(N)T}$}\\ \label{line:ttmc}
% 				\STATE{$\mat{A}^{(i)} \leftarrow \mat{J}_i$ leading left singular vectors of $\mat{Y}_{(i)}$}
% 			\ENDFOR
% 		\UNTIL{the maximum iteration is reached, or the error ceases to decrease;}\\
% 		\STATE{$\T{G} \leftarrow \T{X}\times_1\mat{A}^{(1)T}\times_{2}\mat{A}^{(2)T}\cdots\times_N\mat{A}^{(N)T}$} \label{line:core}
	\end{algorithmic}
\end{algorithm}

The objective function of PARAFAC2 decomposition~\cite{harshman1972parafac2} is given as follows.
\begin{align} \label{eq:PARAFAC2_base}
	\min_{\{\mat{U}_{k}\},\{\mat{S}_{k}\}, \mat{V}} \sum_{k=1}^{K}{|| \mat{X}_{k} - \mat{U}_{k}\mat{S}_{k}\mat{V}^T||_F^2}
	\end{align}
For uniqueness, Harshman~\cite{harshman1972parafac2} imposed the constraint (i.e., $\mat{U}_{k}^T\mat{U} = \Phi$ for all $k$), and
replace $\mat{U}_{k}^T$ with $\mat{Q}_{k}\mat{H}$ where $\mat{Q}_{k}$ is a column orthogonal matrix and $\mat{H}$ is a common matrix for all the slices.
Then, Equation~\eqref{eq:PARAFAC2_base} is reformulated with $\mat{Q}_{k}\mat{H}$:
\begin{align} \label{eq:PARAFAC2_reform}
	\min_{\{\mat{Q}_{k}\},\{\mat{S}_{k}\}, \mat{H}, \mat{V}} \sum_{k=1}^{K}{|| \mat{X}_{k} - \mat{Q}_{k}\mat{H}\mat{S}_{k}\mat{V}^T||_F^2}
	\end{align}
Fig.~\ref{fig:example} shows an example of PARAFAC2 decomposition for a given irregular tensor.
A common approach to solve the above problem is ALS (Alternating Least Square) which iteratively updates a target factor matrix while fixing all factor matrices except for the target.
Algorithm~\ref{alg:als} describes PARAFAC2-ALS.
First, we update each $\mat{Q}_{k}$ while fixing $\mat{H}$, $\mat{V}$, $\mat{S}_{k}$ for $k=1,...,K$ (lines~\ref{alg:line:compute_X1} and~\ref{alg:line:compute_Qk}).
By computing SVD of $\mat{X}_{k}\mat{V}\mat{S}_{k}\mat{H}^T$ as $\mat{Z}'_{k}\mat{\Sigma'}_{k}\mat{P}_{k}^{'T}$, we update $\mat{Q}_{k}$ as $\mat{Z}'_{k}\mat{P}_{k}^{'T}$, which minimizes Equation~\eqref{eq:PARAFAC2_reform} over $\mat{Q}_{k}$~\cite{PerrosPWVSTS17,kiers1999parafac2,golub2013matrix}.
After updating $\mat{Q}_{k}$, the remaining factor matrices $\mat{H}$, $\mat{V}$, $\mat{S}_{k}$ is updated by minimizing the following objective function:
\begin{align} \label{eq:PARAFAC2_reform}
	\min_{\{\mat{S}_{k}\}, \mat{H}, \mat{V}} \sum_{k=1}^{K}{|| \mat{Q}_{k}^T\mat{X}_{k} - \mat{H}\mat{S}_{k}\mat{V}^T||_F^2}
	\end{align}
Minimizing this function is to update $\mat{H}$, $\mat{V}$, $\mat{S}_{k}$ using CP decomposition of a tensor $\T{Y} \in \mathbb{R}^{R \times J \times K}$ whose $k$-th frontal slice is $\mat{Q}_{k}^T\mat{X}_{k}$ (lines~\ref{alg:line:compute_X2} and~\ref{alg:line:construct_y}).
We run a single iteration of CP decomposition for updating them~\cite{kiers1999parafac2} (lines~\ref{alg:line:cp_start} to~\ref{alg:line:update_s}).
$\mat{Q}_{k}$, $\mat{H}$, $\mat{S}_{k}$, and $\mat{V}$ are alternatively updated until convergence.

%\subsection{Problem Definition}
%\label{subsec:problem}

%\begin{problem}[Irregular Tensor Decomposition]
%	\textbf{Given} an irregular dense tensor $\{\mat{X}_{k}\}^{K}_{k=1}$, \textbf{find} the result of PARAFAC2 decomposition of the given tensor.
%\end{problem}

%\begin{problem}[Irregular Tensor Decomposition with Subgroup Query]
%	\textbf{Given} an irregular dense tensor $\{\mat{X}_{k}\}^{K}_{k=1}$ and a subgroup query consisting of a set $\T{K}$ of indices, \textbf{find} the result of PARAFAC2 decomposition of the tensor $\{\mat{X}_{k}\}_{k\in \T{K}}$.
%\end{problem}

Iterative computations with an irregular dense tensor require high computational costs and large intermediate data.
\rdals~\cite{ChengH19} reduces the costs by preprocessing a given tensor and performing PARAFAC2 decomposition using the preprocessed result,
but the improvement of \rdals is limited.
Also, recent works successfully have dealt with sparse irregular tensors by exploiting sparsity.
However, the efficiency of their models depends on the \emph{sparsity} patterns of a given irregular tensor, and thus there is little improvement on irregular \emph{dense} tensors.
Specifically, computations with large dense slices $\mat{X}_{k}$ for each iteration are burdensome as the number of iterations increases.
We focus on improving the efficiency and scalability in irregular dense tensors.

\section{Proposed Method}
\label{sec:method}

In this section, we propose \method, a fast and scalable PARAFAC2 decomposition method for irregular dense tensors.

\subsection{Overview}
\label{subsec:overview}

Before describing main ideas of our method,
we present main challenges that need to be tackled.
\begin{itemize}
	\item[C1.] \textbf{Dealing with large irregular tensors.} PARAFAC2 decomposition (Algorithm~\ref{alg:als}) iteratively updates factor matrices (i.e., $\mat{U}_{k}$, $\mat{S}_{k}$, and $\mat{V}$) using an input tensor.
		Dealing with a large input tensor is burdensome to update the factor matrices as the number of iterations increases.
	\item[C2.] \textbf{Minimizing numerical computations and intermediate data.}
	How can we minimize the intermediate data and overall computations?
	\item[C3.] \textbf{Maximizing multi-core parallelism.} How can we parallelize the computations for PARAFAC2 decomposition?
\end{itemize}

\begin{figure}
	\centering
%	\vspace{-4mm}
	\includegraphics[width=0.45\textwidth]{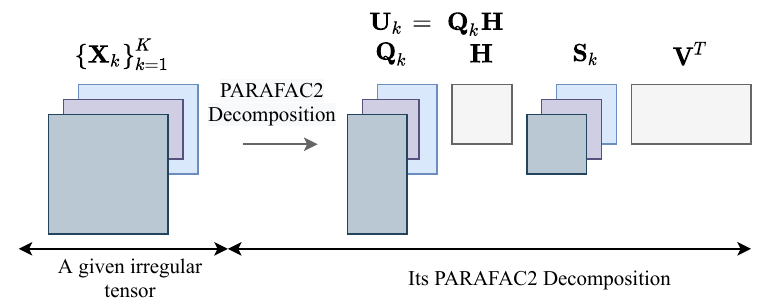}
	\caption{
	Example of PARAFAC2 decomposition.
	Given an irregular tensor $\{ \mat{X}_k \}^{K}_{k=1}$, PARAFAC2 decomposes it into the factor matrices $\mat{H}$, $\mat{V}$, $\mat{Q}_k$, and $\mat{S}_{k}$ for $k=1,..., K$.
	Note that $\mat{Q}_k\mat{H}$ is equal to $\mat{U}_k$.
	}
	\label{fig:example}
\end{figure}

The main ideas that address the challenges mentioned above are as follows:

\begin{itemize}
	\item[I1.] \textbf{Compressing an input tensor using randomized SVD} considerably reduces the computational costs to update factor matrices (Section~\ref{subsec:compression}).
	\item[I2.] \textbf{Careful reordering of computations with the compression results} minimizes the intermediate data and the number of operations (Sections~\ref{subsec:revision} to~\ref{subsec:updating_matrices}).
	\item[I3.] \textbf{Careful distribution of work between threads} enables \method to achieve high efficiency by considering various lengths $I_k$ for $k=1,...,K$ (Section~\ref{subsec:carefulwork}).
\end{itemize}

\begin{figure*}
	\centering
%	\vspace{-4mm}
	\includegraphics[width=0.9\textwidth]{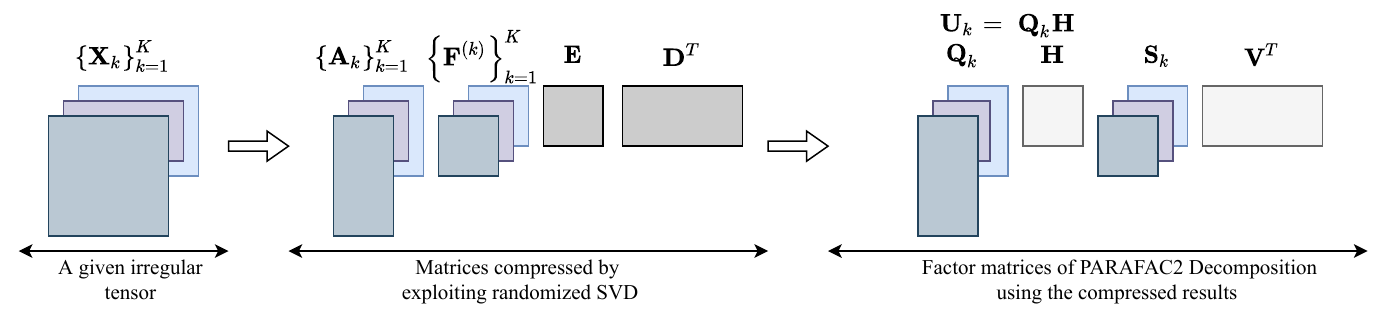}
	\caption{
	Overview of \method.
	Given an irregular tensor $\{ \mat{X}_k \}^{K}_{k=1}$, \method first compresses the given irregular tensor by exploiting randomized SVD.
	Then, \method iteratively and efficiently updates the factor matrices, $\mat{Q}_k$, $\mat{H}$, $\mat{S}_k$, and $\mat{V}$, using only the compressed matrices, to get the result of PARAFAC2 decomposition. 
	}
	\label{fig:overview}
\end{figure*}

As shown in Fig.~\ref{fig:overview}, \method first compresses each slice of an irregular tensor using randomized SVD (Section~\ref{subsec:compression}).
The compression is performed once before iterations, and only the compression results are used at iterations.
It significantly reduces the time and space costs in updating factor matrices.
After compression, \method updates factor matrices at each iteration, by exploiting the compression results (Sections~\ref{subsec:revision} to~\ref{subsec:updating_matrices}).
Careful reordering of computations is required to achieve high efficiency.
Also, by carefully allocating input slices to threads, \method accelerates the overall process (Section~\ref{subsec:carefulwork}).

\vspace{-1mm}
\subsection{Compressing an irregular input tensor}
\label{subsec:compression}
\method (see Algorithm~\ref{alg:dpar2}) is a fast and scalable PARAFAC2 decomposition method based on ALS described in Algorithm~\ref{alg:als}.
The main challenge that needs to be tackled is to minimize the number of heavy computations involved with a given irregular tensor $\{ \mat{X}_k \}^{K}_{k=1}$ consisting of slices $\mat{X}_{k}$ for $k=1,...,K$ (in lines~\ref{alg:line:compute_X1} and~\ref{alg:line:compute_X2} of Algorithm~\ref{alg:als}).
As the number of iterations increases (lines~\ref{alg:line:iter_start} to~\ref{alg:line:iter_end} in Algorithm~\ref{alg:als}), the heavy computations make \als slow.
For efficiency, we preprocess a given irregular tensor into small matrices, and then update factor matrices by carefully using the small ones.
%compute the operations by carefully using the small ones.

Our approach to address the above challenges is to compress a given irregular tensor $\{ \mat{X}_k \}^{K}_{k=1}$ before starting iterations.
% i.e., computing $\mat{X}_{k}\mat{V}\mat{S}_{k}\mat{H}^{T}$ for $k=1,...,K$.
As shown in Fig.~\ref{fig:compression}, our main idea is two-stage lossy compression with randomized SVD for the given tensor:
1) \method performs randomized SVD for each slice $\mat{X}_{k}$ for $k=1,...,K$ at target rank $R$, and
%(i.e., $\T{X}(:,:,k) \in \mathbb{R}^{I_k \times J})$ for $k=1,...,K$ at target rank $R$, and
2) \method performs randomized SVD for a matrix, the horizontal concatenation of singular value matrices and right singular vector matrices of slices $\mat{X}_{k}$.
Randomized SVD allows us to compress slice matrices with low computational costs and low errors.
%$\mat{M}_{(2)} \in \mathbb{R}^{J\times KR}$ of a tensor $\T{M} \in \mathbb{R}^{R \times J \times K}$ for $i=1,...,R$ where the tensor $\T{M}$ is constructed based on collections of singular values and right singular vectors of slice matrices $\mat{X}_{k}$.
%There are two benefits of exploiting randomized SVD.
%\blue{Since compressing the given tensor is performed once and generates small intermediate data, we tackle high computational costs and minimize the intermediate data in updating factor matrices $\mat{H}$, $\mat{Q}_{k}$, $\mat{S}_{k}$, and $\mat{V}$.}
%First, it enables us to efficiently compress a given tensor with low errors since slice matrices of a real-world tensor often have a low dimensional structure.
%Second, exploiting SVD results that are much smaller than the given tensor tackles high computational costs and minimizes the intermediate data in updating factor matrices $\mat{H}$, $\mat{Q}_{k}$, $\mat{S}_{k}$, and $\mat{V}$.
% presents how to compress a given irregular tensor with randomized SVD.

\begin{figure}
	\centering
%	\vspace{-4mm}
	\includegraphics[width=0.49\textwidth]{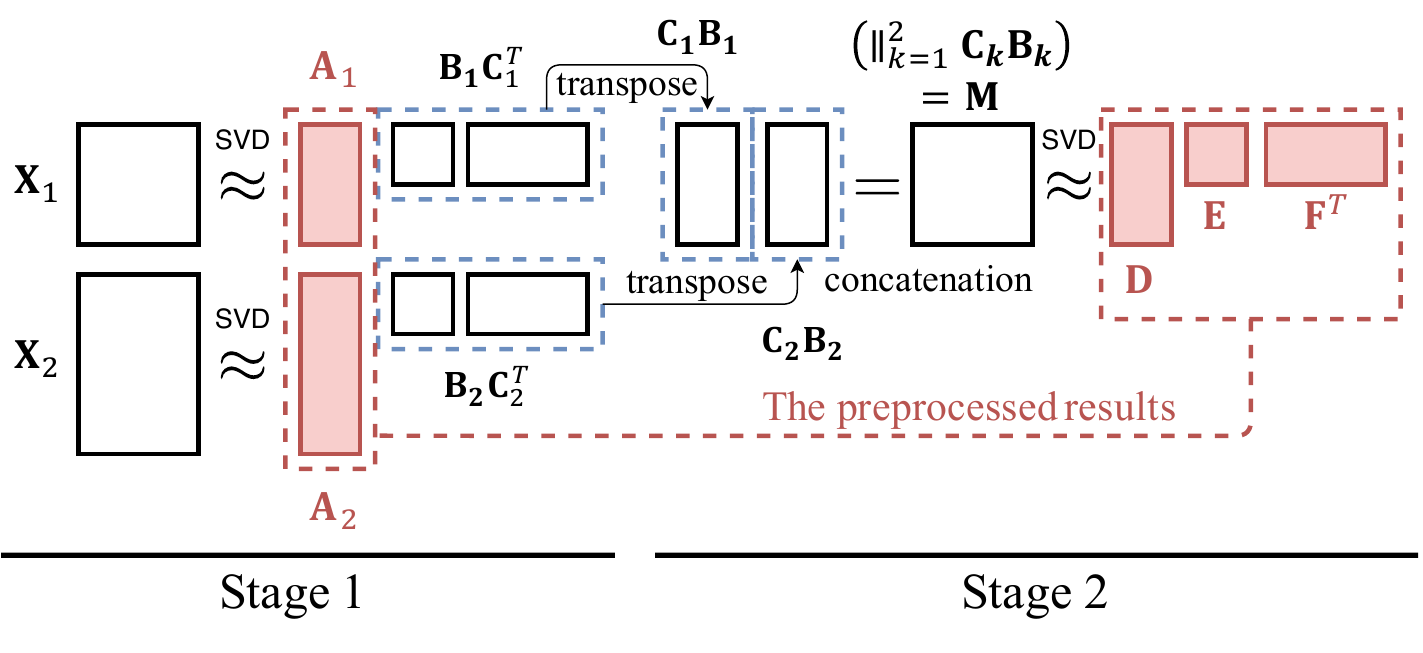}
%	 \subfloat{\includegraphics[width=0.5\textwidth]{FIG/PERF_LEGEND.pdf}} \\
%	 \setcounter{subfigure}{0}
%	 \subfloat[The first stage of compression]{\includegraphics[width=0.49\textwidth]{FIG/COMPRESSION1.pdf}\label{fig:compression1}}
%	 \subfloat[The second stage of compression]{\includegraphics[width=0.235\textwidth]{FIG/COMPRESSION1.pdf}\label{fig:compression1}}	
%	 \subfloat[ML-100k]{\includegraphics[width=0.3\textwidth]{FIG/ML-1m_PERF.pdf}\label{fig:tradeoff_boats_short}}
%	 \subfloat[Fake Figure]{\includegraphics[width=0.3\textwidth]{FIG/ML-100k_PERF.pdf}\label{fig:tradeoff_boats_short}}	 	
	 	 \\
%	 	\vspace{-2mm}
	\caption{
	Two-stage SVD for a given irregular tensor.
%	\textit{mm and trans.} denotes matrix multiplication and transpose.
	In the first stage, \method performs randomized SVD of $\mat{X}_{k}$ for all $k$.
	In the second stage, \method performs randomized SVD of $\mat{M} \in \mathbb{R}^{J\times KR}$ which is the horizontal concatenation of $\mat{C}_k\mat{B}_k$.
%	 after constructing a tensor $\T{M} \in \mathbb{R}^{R \times J \times K}$.
%	\green{See Section~\ref{subsec:exp:querytime} for details.}
	}
	\label{fig:compression}
\end{figure}

\textbf{First Stage.}
%The objective of compressing a given irregular tensor $\{ \mat{X}_k \}^{K}_{k=1}$ is to minimize the costs involved with computations related to the given tensor.
In the first stage, \method compresses a given irregular tensor by performing randomized SVD for each slice $\mat{X}_{k}$ at target rank $R$ (line~\ref{alg2:line:fstageSVD} in Algorithm~\ref{alg:dpar2}).
\begin{align}
\label{eq:fstage_svd}
	\mat{X}_{k} \approx \mat{A}_{k}\mat{B}_{k}\mat{C}_{k}^T
\end{align}
where $\mat{A}_{k} \in \mathbb{R}^{I_k \times R}$ is a matrix consisting of left singular vectors, $\mat{B}_{k} \in \mathbb{R}^{R \times R}$ is a diagonal matrix whose elements are singular values, and $\mat{C}_{k} \in \mathbb{R}^{J \times R}$ is a matrix consisting of right singular vectors.

\textbf{Second Stage.}
Although small compressed data are generated in the first step, there is a room to further compress the intermediate data from the first stage.
In the second stage, we compress a matrix $\mat{M} = \concat_{k=1}^{K}{\left(\mat{C}_{k}\mat{B}_{k}\right)}$ which is the horizontal concatenation of $\mat{C}_{k}\mat{B}_{k}$ for $k=1,...,K$.
Compressing the matrix $\mat{M}$ maximizes the efficiency of updating factor matrices $\mat{H}$, $\mat{V}$, and $\mat{W}$ (see Equation~\eqref{eq:PARAFAC2_reform}) at later iterations.
% where the tensor $\T{M}$ is collections of $\mat{B}_{k}$ and $\mat{C}_{k}$ of slice matrices $\mat{X}_{k}$.
We construct a matrix $\mat{M} \in \mathbb{R}^{J\times KR}$ by horizontally concatenating $\mat{C}_{k}\mat{B}_{k}$ for $k=1,...,K$ (line~\ref{alg2:line:concat_m} in Algorithm~\ref{alg:dpar2}).
Then, \method performs randomized SVD for $\mat{M}$ (line~\ref{alg2:line:sstageSVD} in Algorithm~\ref{alg:dpar2}):
\begin{align}
\label{eq:matricization_svd}
%	\mat{M}_{(2)} = [\mat{C}_{1}\mat{B}_{1}; \cdots ; \mat{C}_{K}\mat{B}_{K}] \approx \mat{D}\mat{E}\mat{F}^{T}
	\mat{M} = [\mat{C}_{1}\mat{B}_{1}; \cdots ; \mat{C}_{K}\mat{B}_{K}] = \concat_{k=1}^{K}{\left(\mat{C}_{k}\mat{B}_{k}\right)} \approx \mat{D}\mat{E}\mat{F}^{T}
\end{align}
where $\mat{D} \in \mathbb{R}^{J \times R}$ is a matrix consisting of left singular vectors, $\mat{E} \in \mathbb{R}^{R \times R}$ is a diagonal matrix whose elements are singular values, and $\mat{F} \in \mathbb{R}^{KR \times R}$ is a matrix consisting of right singular vectors.

With the two stages, we obtain the compressed results $\mat{D}$, $\mat{E}$, $\mat{F}$, and $\mat{A}_{k}$ for $k=1,...,K$.
Before describing how to update factor matrices, we re-express the $k$-th slice $\mat{X}_{k}$ by using the compressed results:
\begin{align}
\label{eq:Xk_final}
	\mat{X}_{k} \approx \mat{A}_{k}\mat{F}^{(k)}\mat{E}\mat{D}^T
\end{align}
where $\mat{F}^{(k)} \in \mathbb{R}^{R\times R}$ is the $k$th vertical block matrix of $\mat{F}$:
\begin{align}
\label{eq:verticalF}
	\mat{F} = \begin{bmatrix} \mat{F}^{(1)} \\ \vdots \\ \mat{F}^{(K)} \end{bmatrix}
\end{align}
Since $\mat{C}_{k}\mat{B}_{k}$ is the $k$th horizontal block of $\mat{M}$ and $\mat{D}\mat{E}\mat{F}^{(k)T}$ is the $k$th horizontal block of $\mat{D}\mat{E}\mat{F}^T$,
$\mat{B}_{k}\mat{C}_{k}^T$ corresponds to $\mat{F}^{(k)}\mat{E}\mat{D}^T$.
Therefore, we obtain Equation~\eqref{eq:Xk_final} by replacing $\mat{B}_{k}\mat{C}_{k}^T$ with $\mat{F}^{(k)}\mat{E}\mat{D}^T$ from Equation~\eqref{eq:fstage_svd}.

In updating factor matrices, we use $\mat{A}_{k}\mat{F}^{(k)}\mat{E}\mat{D}^T$ instead of $\mat{X}_{k}$.
The two-stage compression lays the groundwork for efficient updates.

\subsection{Overview of update rule}
\label{subsec:revision}
%\subsection{Updating factor matrices}
%\label{subsec:updating}

Our goal is to efficiently update factor matrices, $\mat{H}$, $\mat{V}$, and $\mat{S}_{k}$ and $\mat{Q}_{k}$ for $k=1,...,K$, using the compressed results $\mat{A}_{k}\mat{F}^{(k)}\mat{E}\mat{D}^T$.
The main challenge of updating factor matrices is to minimize numerical computations and intermediate data by exploiting the compressed results obtained in Section~\ref{subsec:compression}.
A naive approach would reconstruct $\tilde{\mat{X}}_{k} = \mat{A}_{k}\mat{F}^{(k)}\mat{E}\mat{D}^T$ from the compressed results, and then update the factor matrices.
However, this approach fails to improve the efficiency of updating factor matrices.
% since there is no benefit of compression.
%Therefore, the order of computations is important to achieve extremely high efficiency.
%\method updates factor matrices $\mat{H}$, $\mat{V}$, $\mat{S}_{k}$, and $\mat{Q}_{k}$ without the explicit use of $\mat{X}_{k}$ as well as $\tilde{\mat{X}}_{k}$ reconstructed from $\mat{A}_{k}\mat{F}^{(k)}\mat{E}\mat{D}^T$.
We propose an efficient update rule using the compressed results to 1) find $\mat{Q}_{k}$ and $\mat{Y}_{k}$ (lines~\ref{alg:line:compute_Qk} and~\ref{alg:line:compute_X2} in Algorithm~\ref{alg:als}), and 2) compute a single iteration of CP-ALS (lines~\ref{alg:line:cp_start} to~\ref{alg:line:cp_end} in Algorithm~\ref{alg:als}).

There are two differences between our update rule and PARAFAC2-ALS (Algorithm~\ref{alg:als}).
%There are two differences between our approach and Algorithm~\ref{alg:als}.
First, we avoid explicit computations of $\mat{Q}_{k}$ and $\mat{Y}_{k}$.
Instead, we find small factorized matrices of $\mat{Q}_{k}$ and $\mat{Y}_{k}$, respectively, and then exploit the small ones to update $\mat{H}$, $\mat{V}$, and $\mat{W}$.
The small matrices are computed efficiently by exploiting the compressed results $\mat{A}_{k}\mat{F}^{(k)}\mat{E}\mat{D}^T$ instead of $\mat{X}_{k}$.
%They also support efficient updates of factor matrices.
The second difference is that \method obtains $\mat{H}$, $\mat{V}$, and $\mat{W}$ using the small factorized matrices of $\mat{Y}_{k}$.
Careful ordering of computations with them considerably reduces time and space costs at each iteration.
We describe how to find the factorized matrices of $\mat{Q}_{k}$ and $\mat{Y}_{k}$ in Section~\ref{subsubsec:find_Q}, and how to update factor matrices in Section~\ref{subsec:updating_matrices}. %, respectively.

%We first describe what to modify from Algorithm~\ref{alg:als}.

\subsection{Finding the factorized matrices of $\mat{Q}_{k}$ and $\mat{Y}_{k}$}
\label{subsubsec:find_Q}
The first goal of updating factor matrices is to find the factorized matrices of $\mat{Q}_{k}$ and $\mat{Y}_{k}$ for $k=1,...,K$, respectively.
% re-express $\T{Y}$ using the compressed results $\mat{D}$, $\mat{E}$, and $\mat{F}^{(k)}$ for $k=1,...,K$.
In Algorithm~\ref{alg:als}, finding $\mat{Q}_{k}$ and $\mat{Y}_{k}$ is expensive due to the computations involved with $\mat{X}_{k}$ (lines~\ref{alg:line:compute_X1} and~\ref{alg:line:compute_X2} in Algorithm~\ref{alg:als}).
To reduce the costs for $\mat{Q}_{k}$ and $\mat{Y}_{k}$, our main idea is to exploit the compressed results $\mat{A}_{k}$, $\mat{D}$, $\mat{E}$, and $\mat{F}^{(k)}$, instead of $\mat{X}_{k}$.
Additionally, we exploit the column orthogonal property of $\mat{A}_{k}$, i.e., $\mat{A}_{k}^T \mat{A}_{k} = \mat{I}$, where $\mat{I}$ is the identity matrix.
% of size $R\times R$.
%We obtain equivalent expressions of $\mat{Q}_{k}$ and $\mat{Y}_{k}$ step by step.
%\red{

%Note that the equivalent expression of $\mat{Y}_{k}$ is not explicitly computed, and is exploited in running a single iteration of CP-ALS by replacing $\mat{Y}_{k}$.
%}
% the efficient computations involved with $\mat{X}_{k}$ using the compressed results $\mat{D}$, $\mat{E}$, $\mat{A}_{k}$, and $\mat{F}^{(k)}$ for $k=1,...,K$.
%The first goal of updating factor matrices is the efficient computations involved with $\mat{X}_{k}$ using the compressed results $\mat{D}$, $\mat{E}$, $\mat{A}_{k}$, and $\mat{F}^{(k)}$ for $k=1,...,K$.

%Before finding the equivalent expression of $\mat{Y}_{k}$,
We first re-express $\mat{Q}_{k}$ using the compressed results obtained in Section~\ref{subsec:compression}. % described in lines~\ref{alg:line:compute_X1} and~\ref{alg:line:compute_Qk} of Algorithm~\ref{alg:als}
%For $\mat{Q}_{k}$, we compute $\mat{X}_{k}\mat{V}\mat{S}_{k}\mat{H}^{T}$ and perform SVD of the preceding result.
\method reduces the time and space costs for $\mat{Q}_{k}$ by exploiting the column orthogonal property of $\mat{A}_{k}$.
First, we express $\mat{X}_{k}\mat{V}\mat{S}_{k}\mat{H}^{T}$ as $\mat{A}_{k}\mat{F}^{(k)}\mat{E}\mat{D}^T\mat{V}\mat{S}_{k}\mat{H}^{T}$ by replacing $\mat{X}_{k}$ with $\mat{A}_{k}\mat{F}^{(k)}\mat{E}\mat{D}^T$.
%A approach with the compressed results is to replace $\mat{X}_{k}$ with $\mat{A}_{k}\mat{F}^{(k)}\mat{E}\mat{D}^T$, and then compute the SVD result of $\mat{A}_{k}\mat{F}^{(k)}\mat{E}\mat{D}^T\mat{V}\mat{S}_{k}\mat{H}^{T}$.
%We first re-express $\mat{X}_{k}\mat{V}\mat{S}_{k}\mat{H}^{T}$ as $\mat{A}_{k}\mat{F}^{(k)}\mat{E}\mat{D}^T\mat{V}\mat{S}_{k}\mat{H}^{T}$, and
Next, we need to obtain left and right singular vectors of $\mat{A}_{k}\mat{F}^{(k)}\mat{E}\mat{D}^T$ $\mat{V}\mat{S}_{k}\mat{H}^{T}$.
A naive approach is to compute SVD of $\mat{A}_{k}\mat{F}^{(k)}\mat{E}\mat{D}^T\mat{V}\mat{S}_{k}\mat{H}^{T}$, but there is a more efficient way than this approach.
Thanks to the column orthogonal property of $\mat{A}_{k}$, \method performs SVD of $\mat{F}^{(k)}\mat{E}\mat{D}^T\mat{V}\mat{S}_{k}$ $\mat{H}^{T}$ $ \in \mathbb{R}^{R\times R}$, not $\mat{A}_{k}\mat{F}^{(k)}\mat{E}\mat{D}^T\mat{V}\mat{S}_{k}\mat{H}^{T} \in \mathbb{R}^{I_k\times R}$, at target rank $R$ (line~\ref{alg2:line:compute_X1} in Algorithm~\ref{alg:dpar2}):
\begin{align}
\label{eq:q_k}
	\mat{F}^{(k)}\mat{E}\mat{D}^T\mat{V}\mat{S}_{k}\mat{H}^{T} \overset{\text{SVD}}{=} \mat{Z}_{k}\mat{\Sigma}_{k}\mat{P}_{k}^{T}
\end{align}
where $\mat{\Sigma}_{k}$ is a diagonal matrix whose entries are the singular values of $\mat{F}^{(k)}\mat{E}\mat{D}^T\mat{V}\mat{S}_{k}\mat{H}^{T}$, the column vectors of $\mat{Z}_{k}$ and $\mat{P}_{k}$ are the left singular vectors and the right singular vectors of $\mat{F}^{(k)}\mat{E}\mat{D}^T\mat{V}\mat{S}_{k}\mat{H}^{T}$, respectively.
Then, we obtain the factorized matrices of $\mat{Q}_{k}$ as follows:
\begin{align}
\mat{Q}_{k} = \mat{A}_{k}\mat{Z}_{k}\mat{P}_{k}^{T}
\end{align}
where $\mat{A}_{k}\mat{Z}_{k}$ and $\mat{P}_{k}$ are the left and the right singular vectors of $\mat{A}_{k}\mat{F}^{(k)}\mat{E}\mat{D}^T\mat{V}\mat{S}_{k}\mat{H}^{T}$, respectively.
We avoid the explicit construction of $\mat{Q}_{k}$, and use $\mat{A}_{k}\mat{Z}_{k}\mat{P}_{k}^{T}$ instead of $\mat{Q}_{k}$.
Since $\mat{A}_{k}$ is already column-orthogonal, we avoid performing SVD of $\mat{A}_{k}\mat{F}^{(k)}\mat{E}\mat{D}^T\mat{V}\mat{S}_{k}\mat{H}^{T}$, which are much larger than $\mat{F}^{(k)}\mat{E}\mat{D}^T\mat{V}\mat{S}_{k}\mat{H}^{T}$.

\begin{algorithm} [t]
	\SetNoFillComment
	\caption{\method}
	\label{alg:dpar2}
	\begin{algorithmic} [1]
%		\footnotesize
		\small
		\algsetup{linenosize=\small}

		\renewcommand{\algorithmicrequire}{\textbf{Input:}}
		\renewcommand{\algorithmicensure}{\textbf{Output:}}
		    \REQUIRE $\mat{X}_{k} \in \mathbb{R}^{I_{k} \times J}$ for $k = 1,...,K$
		    \ENSURE $\mat{U}_{k} \in \mathbb{R}^{I_{k} \times R}$, $\mat{S}_{k} \in \mathbb{R}^{R \times R}$ for $k=1,...,K$, and $\mat{V} \in \mathbb{R}^{J \times R}$.
		\renewcommand{\algorithmicrequire}{\textbf{Parameters:}}
		\REQUIRE target rank $R$ \\
		\STATE initialize matrices $\mat{H} \in \mathbb{R}^{R \times R}$, $\mat{V}$, and $\mat{S}_{k}$ for $k=1,...,K$ \\
				\tcc{Compressing slices in parallel}
		\FOR {$k=1,...,K$} \label{alg2:line:start_rand}
			\STATE compute $\mat{A}_{k}\mat{\mat{B}_{k}}\mat{C}_{k}^{T} \leftarrow \text{SVD}(\mat{X}_{k})$ by performing randomized SVD at rank $R$ \label{alg2:line:fstageSVD}  \\
%			\STATE compute $\mat{X}_{k} \approx \mat{A}_{k}\mat{\mat{B}_{k}}\mat{C}_{k}^{T}$ by performing randomized SVD at rank $R$ \label{alg2:line:fstageSVD}  \\
		\ENDFOR		\label{alg2:line:end_rand}
			\STATE $\mat{M} \leftarrow \concat_{k=1}^{K}{\left(\mat{C}_{k}\mat{B}_{k}\right)}$ \label{alg2:line:concat_m} \\
			\STATE compute $\mat{D}\mat{E}\mat{F}^T \leftarrow \text{SVD}(\mat{M})$ by performing randomized SVD at rank $R$ \label{alg2:line:sstageSVD} \\	
%			\STATE compute $\leftarrow \text{\mat{M}} = \concat_{k=1}^{K}{\left(\mat{C}_{k}\mat{B}_{k}\right)} \approx \mat{D}\mat{E}\mat{F}^T$ by performing randomized SVD at rank $R$ \label{alg2:line:sstageSVD} \\	
				\tcc{Iteratively updating factor matrices}			
		\REPEAT \label{alg2:line:iter_start}
		\FOR {$k=1,...,K$} \label{alg2:line:start_iter}
			\STATE compute $\mat{Z}_{k}\mat{\Sigma}_{k}\mat{P}_{k}^{T} \leftarrow \text{SVD}(\mat{F}^{(k)}\mat{E}\mat{D}^T\mat{V}\mat{S}_{k}\mat{H}^{T})$ by performing SVD at rank $R$ \label{alg2:line:compute_X1} \\
%			\STATE $\mat{Q}_{k} \leftarrow \mat{A}_{k}\mat{Z}_{k}\mat{P}_{k}^{T}$ \label{alg:line:compute_Qk} \\
		\ENDFOR \label{alg2:line:end_fiter}\\
%		\Comment*[f]{\textbf{$\mat{E}\mat{D}^T\mat{V}$ is computed once}} \\
		\tcc{no explicit computation of $\mat{Y}_{k}$}				
		\FOR {$k=1,...,K$} \label{alg2:line:start_siter}
			\STATE $\mat{Y}_{k} \leftarrow \mat{P}_{k}\mat{Z}_{k}^T\mat{F}^{(k)}\mat{E}\mat{D}^T$ \label{alg2:line:compute_X2} \\
		\ENDFOR	\label{alg2:line:end_siter}	\\
%		\STATE implicitly construct a tensor $\T{Y} \in \mathbb{R}^{R \times J \times K}$ of slice $\mat{Y}_{k} \in \mathbb{R}^{R \times J}$ for $k=1,...,K$ \label{alg2:line:construct_y} \\
		\tcc{running a single iteration of CP-ALS on $\T{Y}$ }
		\STATE compute $\mat{G}^{(1)} \leftarrow \mat{Y}_{(1)}(\mat{W}\odot \mat{V})$ based on Lemma~\ref{lemma:mttkrp_mode1} \label{alg2:line:compute_g1} \\
		\STATE $\mat{H} \leftarrow \mat{G}^{(1)}(\mat{W}^{T}\mat{W} * 	\mat{V}^{T}\mat{V})^{\dagger}$ \Comment*[f]{\textbf{Normalize $\mat{H}$}}\label{alg2:line:update_h}\\
		\STATE compute $\mat{G}^{(2)} \leftarrow \mat{Y}_{(2)}(\mat{W}\odot \mat{H})$ based on Lemma~\ref{lemma:mttkrp_mode2}  \label{alg2:line:compute_g2} \\
		\STATE $\mat{V} \leftarrow \mat{G}^{(2)}(\mat{W}^{T}\mat{W} * 	\mat{H}^{T}\mat{H})^{\dagger}$ \label{alg2:line:cp_start} \Comment*[f]{\textbf{Normalize $\mat{V}$}} \label{alg2:line:update_v}\\
		\STATE compute $\mat{G}^{(3)} \leftarrow \mat{Y}_{(3)}(\mat{V}\odot \mat{H})$ based on Lemma~\ref{lemma:mttkrp_mode3} \label{alg2:line:compute_g3} \\
		\STATE $\mat{W} \leftarrow \mat{G}^{(3)}(\mat{V}^{T}\mat{V} * 	\mat{H}^{T}\mat{H})^{\dagger}$ \label{alg2:line:update_w}\\				
%		\STATE compute $\mat{H},\mat{V}$, and $\mat{W}$ by running a single iteration of CP-ALS on $\T{Y}$
		\FOR {$k=1,...,K$}
			\STATE $\mat{S}_{k} \leftarrow diag(\mat{W}(k,:))$ \label{alg2:line:update_s1}
		\ENDFOR			\label{alg2:line:update_s} \\
		\UNTIL{the maximum iteration is reached, or the error ceases to decrease;} \label{alg2:line:iter_end}\\
		\FOR {$k=1,...,K$}
			\STATE $\mat{U}_{k} \leftarrow \mat{A}_{k}\mat{Z}_{k}\mat{P}_{k}^{T}\mat{H}$ \label{alg2:line:obtain_U} \\
		\ENDFOR	
%		\STATE \textbf{initialize:} factor matrices $\mat{A}^{(i)}$ ($i=1,...,N$)
% 		\REPEAT
% 			\FOR {$i=1,...,N$}
% 				\STATE{$\T{Y} \leftarrow \T{X}\times_1\mat{A}^{(1)T}\cdots\times_{i-1}\mat{A}^{(i-1)T}\times_{i+1}\mat{A}^{(i+1)T}\cdots\times_N\mat{A}^{(N)T}$}\\ \label{line:ttmc}
% 				\STATE{$\mat{A}^{(i)} \leftarrow \mat{J}_i$ leading left singular vectors of $\mat{Y}_{(i)}$}
% 			\ENDFOR
% 		\UNTIL{the maximum iteration is reached, or the error ceases to decrease;}\\
% 		\STATE{$\T{G} \leftarrow \T{X}\times_1\mat{A}^{(1)T}\times_{2}\mat{A}^{(2)T}\cdots\times_N\mat{A}^{(N)T}$} \label{line:core}
	\end{algorithmic}
\end{algorithm}

Next, we find the factorized matrices of $\mat{Y}_{k}$.
% using $\mat{F}^{(k)}$, $\mat{E}$, $\mat{D}$, and $\mat{Q}_{k}$.
\method re-expresses $\mat{Q}_{k}^T\mat{X}_{k}$ (line~\ref{alg:line:compute_X2} in Algorithm~\ref{alg:als}) as $\mat{Q}_{k}^T\mat{A}_{k}\mat{F}^{(k)}\mat{E}\mat{D}^T$ using Equation~\eqref{eq:Xk_final}.
Instead of directly computing $\mat{Q}_{k}^T\mat{A}_{k}\mat{F}^{(k)}\mat{E}\mat{D}^T$,
we replace $\mat{Q}_{k}^T$ with $\mat{P}_{k}\mat{Z}_{k}^T\mat{A}_{k}^T$.
Then, we represent $\mat{Y}_{k}$ as the following expression (line~\ref{alg2:line:compute_X2} in Algorithm~\ref{alg:dpar2}):
\begin{align*}
	\mat{Y}_{k} \leftarrow  & \mat{Q}_{k}^T\mat{A}_{k}\mat{F}^{(k)}\mat{E}\mat{D}^T
	  = \mat{P}_{k}\mat{Z}_{k}^T\mat{A}_{k}^T\mat{A}_{k}\mat{F}^{(k)}\mat{E}\mat{D}^T\\
	&= \mat{P}_{k}\mat{Z}_{k}^T\mat{F}^{(k)}\mat{E}\mat{D}^T
\end{align*}
Note that we use the property $\mat{A}_{k}^T\mat{A}_{k} = \mat{I}_{R \times R}$, where $\mat{I}_{R \times R}$ is the identity matrix of size $R\times R$,
for the last equality.
By exploiting the factorized matrices of $\mat{Q}_{k}$, we compute $\mat{Y}_{k}$ without involving $\mat{A}_{k}$ in the process.
%Note that there is no computation for $\mat{Y}_{k}$ after obtaining $\mat{Q}_{k}$, and we use $\mat{P}_{k}\mat{Z}_{k}^T\mat{F}^{(k)}\mat{E}\mat{D}^T$ for update instead of $\mat{Y}_{k}$.

\subsection{Updating $\mat{H}$, $\mat{V}$, and $\mat{W}$}
\label{subsec:updating_matrices}
The next goal is to efficiently update the matrices $\mat{H}$, $\mat{V}$, and $\mat{W}$ using the small factorized matrices of $\mat{Y}_{k}$.
Naively, we would compute $\T{Y}$ and run a single iteration of CP-ALS with $\T{Y}$ to update $\mat{H}$, $\mat{V}$, and $\mat{W}$ (lines~\ref{alg:line:cp_start} to~\ref{alg:line:cp_end} in Algorithm~\ref{alg:als}).
However, multiplying a matricized tensor and a Khatri-Rao product (e.g., $\mat{Y}_{(1)}(\mat{W}\odot \mat{V})$) is burdensome, and thus we exploit the structure of the decomposed results $\mat{P}_{k}\mat{Z}_{k}^T\mat{F}^{(k)}\mat{E}\mat{D}^T$ of $\mat{Y}_{k}$ to reduce memory requirements and computational costs.
In other word, we do not compute $\mat{Y}_{k}$, and use only $\mat{P}_{k}\mat{Z}_{k}^T\mat{F}^{(k)}\mat{E}\mat{D}^T$ in updating $\mat{H}$, $\mat{V}$, and $\mat{W}$.
Note that the $k$-th frontal slice of $\T{Y}$, $\T{Y}(:,:,k)$, is $\mat{P}_{k}\mat{Z}_{k}^T\mat{F}^{(k)}\mat{E}\mat{D}^T$.
%To be efficient, we do not explicitly compute $\mat{Y}_{k}$ for $k=1,...,K$, and use $\mat{P}_{k}\mat{Z}_{k}^T\mat{F}^{(k)}\mat{E}\mat{D}^T$ term to update the factor matrices.

\begin{figure}
\centering
	\includegraphics[width=0.45\textwidth]{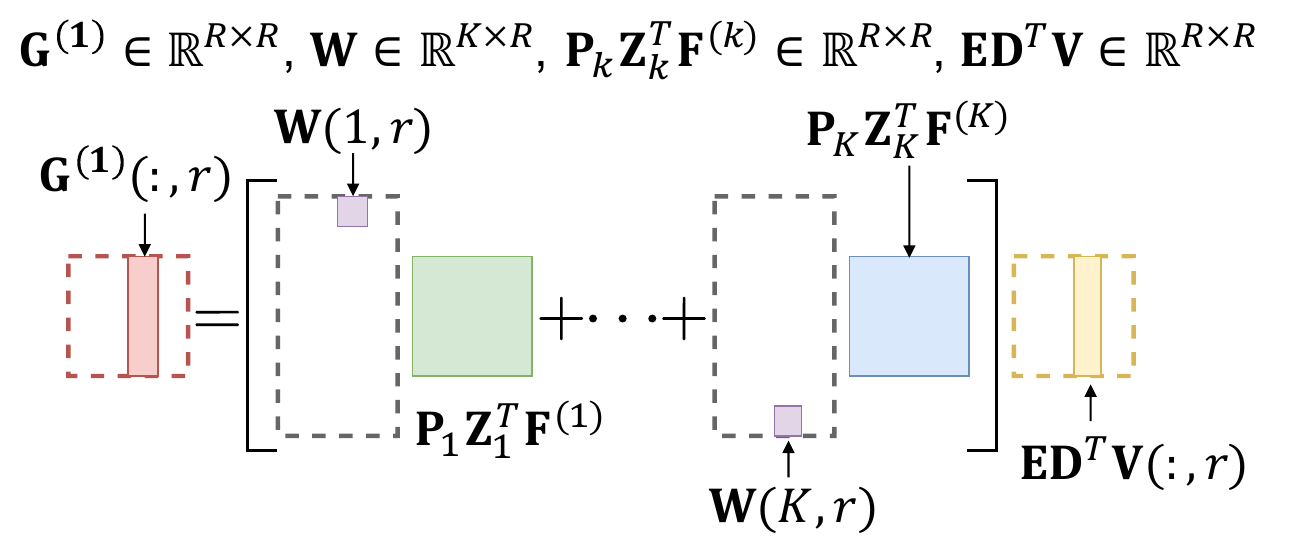}
	\vspace{-2mm}
\caption{
Computation for $\mat{G}^{(1)} = \mat{Y}_{(1)}(\mat{W}\odot \mat{V})$.
The $r$th column $\mat{G}^{(1)}(:,r)$ of $\mat{G}^{(1)}$ is computed by $\left(\sum_{k=1}^{K}{\mat{W}(k,r)\left(\mat{P}_{k}\mat{Z}_{k}^T\mat{F}^{(k)}\right)}\right)\mat{E}\mat{D}^T\mat{V}(:,r)$.
	}
\label{fig:example_g1}	
\end{figure}

\textbf{Updating $\mat{H}$.}
In $\mat{Y}_{(1)}(\mat{W}\odot \mat{V})(\mat{W}^{T}\mat{W} * \mat{V}^{T}\mat{V})^{\dagger}$, we focus on efficiently computing $\mat{Y}_{(1)}(\mat{W} \odot \mat{V})$
based on Lemma~\ref{lemma:mttkrp_mode1}.
A naive computation for $\mat{Y}_{(1)}(\mat{W} \odot \mat{V})$ requires a high computational cost $\T{O}(JKR^2)$ due to the explicit reconstruction of $\mat{Y}_{(1)}$.
Therefore, we compute that term without the reconstruction by carefully determining the order of computations and exploiting the factorized matrices of $\mat{Y}_{(1)}$, $\mat{D}$, $\mat{E}$, $\mat{P}_{k}$, $\mat{Z}_{k}$, and $\mat{F}^{(k)}$ for $k=1,...,K$.
With Lemma~\ref{lemma:mttkrp_mode1}, we reduce the computational cost of $\mat{Y}_{(1)}(\mat{W} \odot \mat{V})$ to $\T{O}(JR^2 + KR^3)$.

\begin{lemma}
\label{lemma:mttkrp_mode1}
Let us denote $\mat{Y}_{(1)}(\mat{W}\odot \mat{V})$ with $\mat{G}^{(1)} \in \mathbb{R}^{R \times R}$.
$\mat{G}^{(1)}(:,r)$ is equal to $\left(\left(\sum_{k=1}^{K}{\mat{W}(k,r)\left(\mat{P}_{k}\mat{Z}_{k}^T\mat{F}^{(k)}\right)}\right)\mat{E}\mat{D}^T\mat{V}(:,r)\right)$.
%$\left(\concat_{k=1}^{K}{\left(\mat{P}_{k}\mat{Z}_{k}^T\mat{F}^{(k)}\right)}\right) \left(\mat{W}(:,r) \otimes \mat{E}\mat{D}^T\mat{V}(:,r)\right)$.
\QEDB
%$\mat{G}^{(1)}(:,r)$ is equal to $\begin{bmatrix} \mat{P}_{1}\mat{Z}_{1}^T\mat{F}^{(1)} &; \cdots ;&  \mat{P}_{K}\mat{Z}_{K}^T\mat{F}^{(K)} \end{bmatrix} \left(\mat{W}(:,r) \otimes \mat{E}\mat{D}^T\mat{V}(:,r)\right)$.
%The $(i, j)$-th entry of $\mat{G}^{(1)}$ is
%$(\mat{W}(i,:))^T\mat{P}_{k}\mat{Z}_{k}^T\mat{F}^{(k)}\mat{E}\mat{D}^T\mat{V}(j,:)$.	
\end{lemma}

\begin{proof}
$\mat{Y}_{(1)}$ is represented as follows:
	\begin{align*}
		 \mat{Y}_{(1)} &= \begin{bmatrix} \mat{P}_{1}\mat{Z}_{1}^T\mat{F}^{(1)}\mat{E}\mat{D}^T &; \cdots ;&  \mat{P}_{K}\mat{Z}_{K}^T\mat{F}^{(K)}\mat{E}\mat{D}^T \end{bmatrix} \\
		& = \left(\concat_{k=1}^{K}{\left(\mat{P}_{k}\mat{Z}_{k}^T\mat{F}^{(k)}\right)}\right)
		\begin{bmatrix}\mat{E}\mat{D}^T & \cdots & \mat{O}
		\\
		\vdots &  \ddots & \vdots
		\\
		\mat{O}  & \cdots &  \mat{E}\mat{D}^T
		\end{bmatrix} \\
		& = \left(\concat_{k=1}^{K}{\left(\mat{P}_{k}\mat{Z}_{k}^T\mat{F}^{(k)}\right)}\right) \left(\mat{I}_{K\times K} \otimes \mat{E}\mat{D}^T\right)
	\end{align*}
where $\mat{I}_{K \times K}$ is the identity matrix of size $K\times K$.
Then, $\mat{G}^{(1)} = \mat{Y}_{(1)}(\mat{W} \odot \mat{V})$ is expressed as follows:
\begin{align*}
	\mat{G}^{(1)} &= \left(\concat_{k=1}^{K}{\left(\mat{P}_{k}\mat{Z}_{k}^T\mat{F}^{(k)}\right)}\right) \\
	&\times \left(\mat{I}_{K\times K} \otimes \mat{E}\mat{D}^T\right)\left(\concat_{r=1}^{R}{\left(\mat{W}(:,r)\otimes \mat{V}(:,r)\right)}\right)\\
	& = \left(\concat_{k=1}^{K}{\left(\mat{P}_{k}\mat{Z}_{k}^T\mat{F}^{(k)}\right)}\right)
	\left(\concat_{r=1}^{R}{\left(\mat{W}(:,r)\otimes  \mat{E}\mat{D}^T\mat{V}(:,r)\right)}\right)
\end{align*}
The mixed-product property (i.e., $(\mat{A}\otimes \mat{B})(\mat{C}\otimes \mat{D}) = \mat{A}\mat{C}\otimes \mat{B}\mat{D})$) is used in the above equation.
Therefore, $\mat{G}^{(1)}(:,r)$ is equal to $\left(\concat_{k=1}^{K}{\left(\mat{P}_{k}\mat{Z}_{k}^T\mat{F}^{(k)}\right)}\right) \left(\mat{W}(:,r) \otimes \mat{E}\mat{D}^T\mat{V}(:,r)\right)$.
We represent it as $\sum_{k=1}^{K}{\mat{W}(k,r)\left(\mat{P}_{k}\mat{Z}_{k}^T\mat{F}^{(k)}\right)}$ $\mat{E}\mat{D}^T\mat{V}(:,r)$ using block matrix multiplication
since the $k$-th vertical block vector of $ \left(\mat{W}(:,r) \otimes \mat{E}\mat{D}^T\mat{V}(:,r)\right)\in \mathbb{R}^{KR}$ is $\mat{W}(k,r)\mat{E}\mat{D}^T\mat{V}(:,r) \in \mathbb{R}^{R}$.
\end{proof}
\vspace{-2mm}
%\begin{proof}
%	See Appendix~\ref{proof:mttkrp_mode1}.
%\end{proof}
As shown in Fig.~\ref{fig:example_g1}, we compute $\mat{Y}_{(1)}(\mat{W} \odot \mat{V})$ column by column.
In computing $\mat{G}^{(1)}(:,r)$, we compute $\mat{E}\mat{D}^T\mat{V}(:,r)$, sum up $\mat{W}(k,r)\left(\mat{P}_{k}\mat{Z}_{k}^T\mat{F}^{(k)}\right)$ for all $k$, and then perform matrix multiplication between the two preceding results (line~\ref{alg2:line:compute_g1} in Algorithm~\ref{alg:dpar2}).
After computing $\mat{G}^{(1)} \leftarrow \mat{Y}_{(1)}(\mat{W}\odot \mat{V})$, we update $\mat{H}$ by computing $\mat{G}^{(1)}(\mat{W}^{T}\mat{W} * \mat{V}^{T}\mat{V})^{\dagger}$ where ${\dagger}$ denotes the Moore-Penrose pseudoinverse (line~\ref{alg2:line:update_h} in Algorithm~\ref{alg:dpar2}).
Note that the pseudoinverse operation requires a lower computational cost compared to computing $\mat{G}^{(1)}$ since the size of $(\mat{W}^{T}\mat{W} * \mat{V}^{T}\mat{V}) \in \mathbb{R}^{R\times R}$ is small.

\begin{figure}
\centering
	\includegraphics[width=0.48\textwidth]{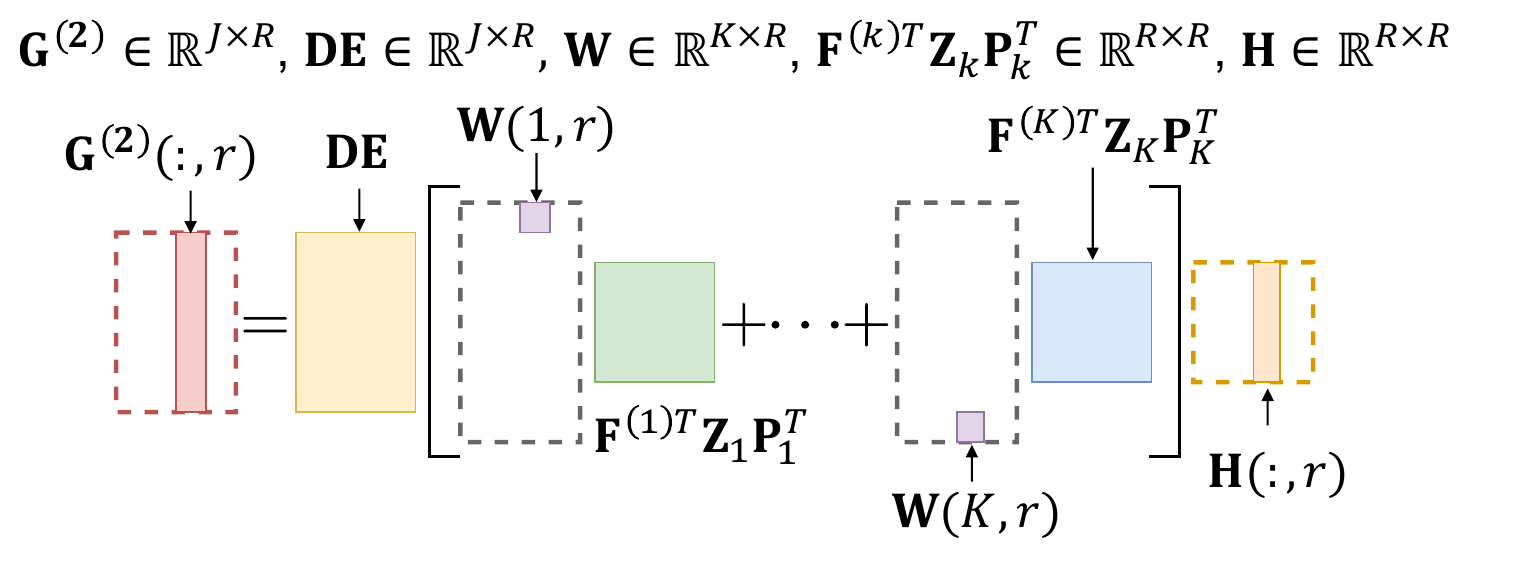}
		\vspace{-2mm}
\caption{
Computation for $\mat{G}^{(2)} = \mat{Y}_{(2)}(\mat{W}\odot \mat{H})$.
The $r$th column $\mat{G}^{(2)}(:,r)$ of $\mat{G}^{(2)}$ is computed by $\mat{D}\mat{E}\sum_{k=1}^{K}{\left(\mat{W}(k,r)\mat{F}^{(k)T}\mat{Z}_{k}\mat{P}_{k}^{T}\mat{H}(:,r)\right)}$.}
\label{fig:example_g2}
\end{figure}

\textbf{Updating $\mat{V}$.}
In computing $\mat{Y}_{(2)}(\mat{W}\odot \mat{U})(\mat{W}^{T}\mat{W} * \mat{U}^{T}\mat{U})^{\dagger}$, we need to efficiently compute $\mat{Y}_{(2)}(\mat{W} \odot \mat{U})$
based on Lemma~\ref{lemma:mttkrp_mode2}.
As in updating $\mat{H}$, a naive computation for $\mat{Y}_{(2)}(\mat{W} \odot \mat{U})$ requires a high computational cost $\T{O}(JKR^2)$.
We efficiently compute $\mat{Y}_{(2)}(\mat{W} \odot \mat{U})$ with the cost $\T{O}(JR^2 + KR^3)$, by carefully determining the order of computations and exploiting the factorized matrices of $\mat{Y}_{(2)}$.
%The detail of the computation is described in Lemma~\ref{lemma:mttkrp_mode2}.
\begin{lemma}
\label{lemma:mttkrp_mode2}
Let us denote $\mat{Y}_{(2)}(\mat{W}\odot \mat{H})$ with $\mat{G}^{(2)} \in \mathbb{R}^{J \times R}$.
$\mat{G}^{(2)}(:,r)$ is equal to $\mat{D}\mat{E}\left(\sum_{k=1}^{K}{\left(\mat{W}(k,r)\mat{F}^{(k)T}\mat{Z}_{k}\mat{P}_{k}^{T}\mat{H}(:,r)\right)}\right)$.
\QEDB
%$\mat{G}^{(1)}(:,r)$ is equal to $\begin{bmatrix} \mat{P}_{1}\mat{Z}_{1}^T\mat{F}^{(1)} &; \cdots ;&  \mat{P}_{K}\mat{Z}_{K}^T\mat{F}^{(K)} \end{bmatrix} \left(\mat{W}(:,r) \otimes \mat{E}\mat{D}^T\mat{V}(:,r)\right)$.
%The $(i, j)$-th entry of $\mat{G}^{(1)}$ is
%$(\mat{W}(i,:))^T\mat{P}_{k}\mat{Z}_{k}^T\mat{F}^{(k)}\mat{E}\mat{D}^T\mat{V}(j,:)$.	
\end{lemma}

\begin{proof}
		$\mat{Y}_{(2)}$ is represented as follows:
	\begin{align*}
		 \mat{Y}_{(2)} &= \begin{bmatrix} \mat{D}\mat{E}\mat{F}^{(1)T}\mat{Z}_{1}\mat{P}_{1}^{T} &; \cdots ;&  \mat{D}\mat{E}\mat{F}^{(K)T}\mat{Z}_{K}\mat{P}_{K}^{T} \end{bmatrix} \\
%		& = \mat{D}\mat{E}
%		\begin{bmatrix}\mat{F}^{(1)T}\mat{Z}_{1}\mat{P}_{1}^{T} & \cdots & \mat{F}^{(K)T}\mat{Z}_{K}\mat{P}_{K}^{T}
%		\end{bmatrix} \\
		&= \mat{D}\mat{E}\left(\concat_{k=1}^{K}{\mat{F}^{(k)T}\mat{Z}_{k}\mat{P}_{k}^{T}}\right)
	\end{align*}
Then, $\mat{G}^{(2)} = \mat{Y}_{(2)}(\mat{W} \odot \mat{H})$ is expressed as follows:
\begin{align*}
\mat{G}^{(2)} &= \mat{D}\mat{E}\left(\concat_{k=1}^{K}{\mat{F}^{(k)T}\mat{Z}_{k}\mat{P}_{k}^{T}}\right)\\
&\times \begin{bmatrix}\mat{W}(1,1)\mat{H}(:,1) ;& \cdots &; \mat{W}(1,R)\mat{H}(:,R)
		\\
		\vdots &  \vdots & \vdots
		\\
		\mat{W}(K,1)\mat{H}(:,1) ; & \cdots & ; \mat{W}(K,R)\mat{H}(:,R)
		\end{bmatrix}
%	\mat{G}^{(1)} &= \left(\concat_{k=1}^{K}{\left(\mat{P}_{k}\mat{Z}_{k}^T\mat{F}^{(k)}\right)}\right)
% \left(\mat{I}_{K\times K} \otimes \mat{E}\mat{D}^T\right)\left(\concat_{r=1}^{R}{\left(\mat{W}(:,r)\otimes \mat{V}(:,r)\right)}\right)\\
%	& = \left(\concat_{k=1}^{K}{\left(\mat{P}_{k}\mat{Z}_{k}^T\mat{F}^{(k)}\right)}\right)
%	\left(\concat_{r=1}^{R}{\left(\mat{W}(:,r)\otimes  \mat{E}\mat{D}^T\mat{V}(:,r)\right)}\right)
\end{align*}
$\mat{G}^{(2)}(:,r)$ is equal to $\mat{D}\mat{E}\sum_{k=1}^{K}{\left(\mat{W}(k,r)\mat{F}^{(k)T}\mat{Z}_{k}\mat{P}_{k}^{T}\mat{H}(:,r)\right)}$ according to the above equation.	
\end{proof}
\vspace{-2mm}
\noindent As shown in Fig.~\ref{fig:example_g2}, we compute $\mat{G}^{(2)} \leftarrow \mat{Y}_{(2)}(\mat{W}\odot \mat{H})$ column by column.
After computing $\mat{G}^{(2)}$, we update $\mat{V}$ by computing $\mat{G}^{(2)}(\mat{W}^{T}\mat{W} * \mat{H}^{T}\mat{H})^{\dagger}$ (lines~\ref{alg2:line:compute_g2} and~\ref{alg2:line:update_v} in Algorithm~\ref{alg:dpar2}).

% where ${\dagger}$ denotes the Moore-Penrose pseudoinverse.

\textbf{Updating $\mat{W}$.}
In computing $\mat{Y}_{(3)}(\mat{V}\odot \mat{H})(\mat{V}^{T}\mat{V} * \mat{H}^{T}\mat{H})^{\dagger}$, we efficiently compute $\mat{Y}_{(3)}(\mat{V} \odot \mat{H})$
based on Lemma~\ref{lemma:mttkrp_mode3}.
As in updating $\mat{H}$ and $\mat{V}$, a naive computation for $\mat{Y}_{(3)}(\mat{V} \odot \mat{H})$ requires a high computational cost $\T{O}(JKR^2)$.
We compute $\mat{Y}_{(3)}(\mat{V} \odot \mat{H})$ with the cost $\T{O}(JR^2 + KR^3)$ based on Lemma~\ref{lemma:mttkrp_mode3}.
Exploiting the factorized matrices of $\mat{Y}_{(3)}$ and carefully determining the order of computations improves the efficiency.

\begin{figure}
\centering
	\includegraphics[width=0.45\textwidth]{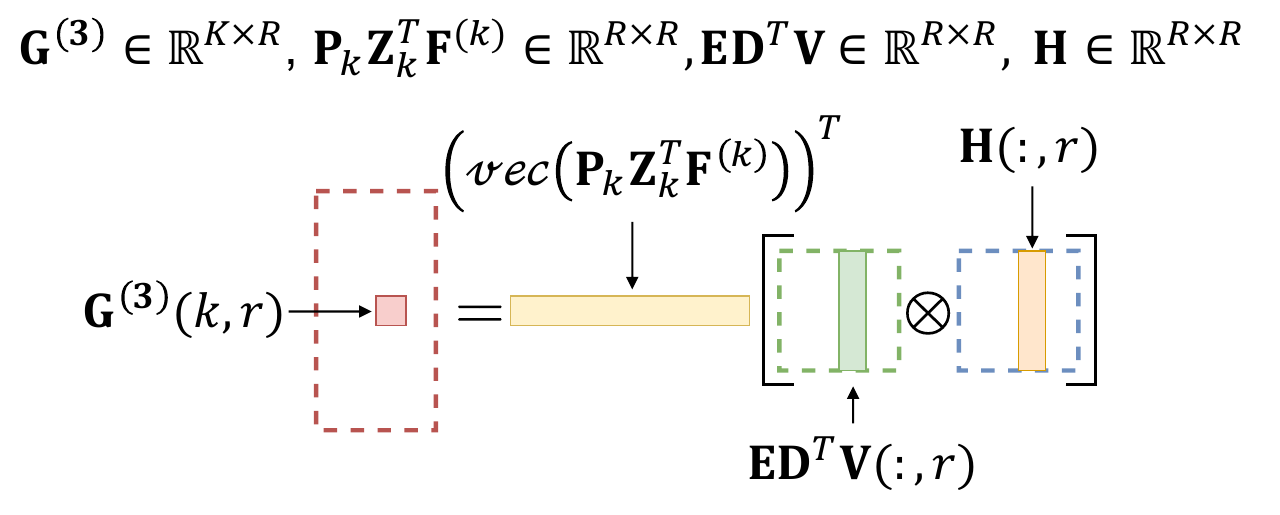}
\caption{
Computation for $\mat{G}^{(3)} = \mat{Y}_{(3)}(\mat{V} \odot \mat{H})$.
$\mat{G}^{(3)}(k,r)$ is computed by $\left(vec\left(\mat{P}_{k}\mat{Z}_{k}^T\mat{F}^{(k)}\right)\right)^T\left(\mat{E}\mat{D}^T\mat{V}(:,r)\otimes \mat{H}(:,r)\right)$.
	}
\label{fig:example_g3}	
\end{figure}

\begin{lemma}
\label{lemma:mttkrp_mode3}
Let us denote $\mat{Y}_{(3)}(\mat{V}\odot \mat{H})$ with $\mat{G}^{(3)} \in \mathbb{R}^{K \times R}$.
$\mat{G}^{(3)}(k,r)$ is equal to $\left(vec\left(\mat{P}_{k}\mat{Z}_{k}^T\mat{F}^{(k)}\right)\right)^T\left(\mat{E}\mat{D}^T\mat{V}(:,r)\otimes \mat{H}(:,r)\right)$ where $vec(\cdot)$ denotes the vectorization of a matrix.
\QEDB
%$\mat{G}^{(1)}(:,r)$ is equal to $\begin{bmatrix} \mat{P}_{1}\mat{Z}_{1}^T\mat{F}^{(1)} &; \cdots ;&  \mat{P}_{K}\mat{Z}_{K}^T\mat{F}^{(K)} \end{bmatrix} \left(\mat{W}(:,r) \otimes \mat{E}\mat{D}^T\mat{V}(:,r)\right)$.
%The $(i, j)$-th entry of $\mat{G}^{(1)}$ is
%$(\mat{W}(i,:))^T\mat{P}_{k}\mat{Z}_{k}^T\mat{F}^{(k)}\mat{E}\mat{D}^T\mat{V}(j,:)$.	
\end{lemma}

\begin{proof}
	$\mat{Y}_{(3)}$ is represented as follows:
	\begin{align*}
		 \mat{Y}_{(3)}
		  &= \begin{bmatrix} \left(vec\left(\mat{P}_{1}\mat{Z}_{1}^T\mat{F}^{(1)}\mat{E}\mat{D}^T\right)\right)^T \\ \vdots \\  \left(vec\left(\mat{P}_{K}\mat{Z}_{K}^T\mat{F}^{(K)}\mat{E}\mat{D}^T\right)\right)^T \end{bmatrix} \\
		 & = \left(\concat_{k=1}^{K}{\left(vec\left(\mat{P}_{k}\mat{Z}_{k}^T\mat{F}^{(k)}\mat{E}\mat{D}^T\right)\right)}\right)^T \\
		  & = \left(\concat_{k=1}^{K}{\left(\mat{D}\mat{E}\otimes \mat{I}\right)vec\left(\mat{P}_{k}\mat{Z}_{k}^T\mat{F}^{(k)}\right)}\right)^T \\
		  & = \left(\concat_{k=1}^{K}{\left(vec\left(\mat{P}_{k}\mat{Z}_{k}^T\mat{F}^{(k)}\right)\right)}\right)^T\left(\mat{E}\mat{D}^T \otimes \mat{I}_{R\times R}\right)
%		& = \left(\concat_{k=1}^{K}{\left(\mat{P}_{k}\mat{Z}_{k}^T\mat{F}^{(k)}\right)}\right)
%		\begin{bmatrix}\mat{E}\mat{D}^T & \cdots & \mat{O}
%		\\
%		\vdots &  \ddots & \vdots
%		\\
%		\mat{O}  & \cdots &  \mat{E}\mat{D}^T
%		\end{bmatrix}
%		& = \left(\concat_{k=1}^{K}{\left(\mat{P}_{k}\mat{Z}_{k}^T\mat{F}^{(k)}\right)}\right)
%		\begin{bmatrix}\mat{E}\mat{D}^T & \cdots & \mat{O}
%		\\
%		\vdots &  \ddots & \vdots
%		\\
%		\mat{O}  & \cdots &  \mat{E}\mat{D}^T
%		\end{bmatrix} \\
%
%
%		& =
%		\begin{bmatrix} \left(\left(\mat{D}\mat{E}\otimes \mat{I}\right)vec\left(\mat{P}_{1}\mat{Z}_{1}^T\mat{F}^{(1)}\right)\right)^T \\ \vdots \\  \left(\left(\mat{D}\mat{E}\otimes \mat{I}\right)vec\left(\mat{P}_{K}\mat{Z}_{K}^T\mat{F}^{(K)}\right)\right)^T \end{bmatrix} \\
%		& =
%		\begin{bmatrix} \left(vec\left(\mat{P}_{1}\mat{Z}_{1}^T\mat{F}^{(1)}\right)\right)^T \\ \vdots \\  \left(vec\left(\mat{P}_{K}\mat{Z}_{K}^T\mat{F}^{(K)}\right)\right)^T \end{bmatrix}\left(\mat{E}\mat{D}^T \otimes \mat{I}_{R\times R}\right)
	\end{align*}
where $\mat{I}_{R \times R}$ is the identity matrix of size $R \times R$.
The property of the vectorization (i.e., $vec(\mat{A}\mat{B}) = (\mat{B}^T \otimes \mat{I})vec(\mat{A})$) is used.
Then, $\mat{G}^{(3)} = \mat{Y}_{(3)}(\mat{V} \odot \mat{H})$ is expressed as follows:
\begin{align*}
\mat{G}^{(3)}
%&= \begin{bmatrix} \left(vec\left(\mat{P}_{1}\mat{Z}_{1}^T\mat{F}^{(1)}\right)\right)^T \\ \vdots \\  \left(vec\left(\mat{P}_{K}\mat{Z}_{K}^T\mat{F}^{(K)}\right)\right)^T \end{bmatrix}\left(\mat{E}\mat{D}^T \otimes \mat{I}_{R\times R}\right)\left(\concat_{r=1}^{R}{\left(\mat{V}(:,r)\otimes \mat{H}(:,r)\right)}\right) \\
 	&=
% 	\begin{bmatrix} \left(vec\left(\mat{P}_{1}\mat{Z}_{1}^T\mat{F}^{(1)}\right)\right)^T \\ \vdots \\  \left(vec\left(\mat{P}_{K}\mat{Z}_{K}^T\mat{F}^{(K)}\right)\right)^T \end{bmatrix} \\
 	\left(\concat_{k=1}^{K}{\left(vec\left(\mat{P}_{k}\mat{Z}_{k}^T\mat{F}^{(k)}\right)\right)}\right)^T\\
 	&\times \left(\concat_{r=1}^{R}{\left(\mat{E}\mat{D}^T\mat{V}(:,r)\otimes \mat{H}(:,r)\right)}\right)
%	\mat{G}^{(1)} &= \left(\concat_{k=1}^{K}{\left(\mat{P}_{k}\mat{Z}_{k}^T\mat{F}^{(k)}\right)}\right)
% \left(\mat{I}_{K\times K} \otimes \mat{E}\mat{D}^T\right)\left(\concat_{r=1}^{R}{\left(\mat{W}(:,r)\otimes \mat{V}(:,r)\right)}\right)\\
%	& = \left(\concat_{k=1}^{K}{\left(\mat{P}_{k}\mat{Z}_{k}^T\mat{F}^{(k)}\right)}\right)
%	\left(\concat_{r=1}^{R}{\left(\mat{W}(:,r)\otimes  \mat{E}\mat{D}^T\mat{V}(:,r)\right)}\right)
\end{align*}
$\mat{G}^{(3)}(k,r)$ is $\left(vec\left(\mat{P}_{k}\mat{Z}_{k}^T\mat{F}^{(k)}\right)\right)^T\left(\mat{E}\mat{D}^T\mat{V}(:,r)\otimes \mat{H}(:,r)\right)$ according to the above equation.
\end{proof}
\vspace{-2mm}

%\begin{proof}
%	See Appendix~\ref{proof:mttkrp_mode3}.
%\end{proof}
\noindent We compute $\mat{G}^{(3)} = \mat{Y}_{(3)}(\mat{V}\odot \mat{H})$ row by row.
Fig.~\ref{fig:example_g3} shows how we compute $\mat{G}^{(3)}(k,r)$.
In computing $\mat{G}^{(3)}$, we first compute $\mat{E}\mat{D}^T\mat{V}$, and then obtain $\mat{G}^{(3)}(k,:)$ for all $k$ (line~\ref{alg2:line:compute_g3} in Algorithm~\ref{alg:dpar2}).
After computing $\mat{G}^{(3)}$, we update $\mat{W}$ by computing $\mat{G}^{(3)}(\mat{V}^{T}\mat{V} * \mat{H}^{T}\mat{H})^{\dagger}$ where ${\dagger}$ denotes the Moore-Penrose pseudoinverse (line~\ref{alg2:line:update_w} in Algorithm~\ref{alg:dpar2}).
We obtain $\mat{S}_k$ whose diagonal elements correspond to the $k$th row vector of $\mat{W}$ (line~\ref{alg2:line:update_s1} in Algorithm~\ref{alg:dpar2}).

After convergence, we obtain the factor matrices, ($\mat{U}_{k} \leftarrow \mat{A}_{k}\mat{Z}_{k}\mat{P}_{k}^T\mat{H}$ $ = \mat{Q}_{k}\mat{H}$), $\mat{S}_{k}$, and $\mat{V}$ (line~\ref{alg2:line:obtain_U} in Algorithm~\ref{alg:dpar2}).

\textbf{Convergence Criterion.}
At the end of each iteration, we determine whether to stop or not (line~\ref{alg2:line:iter_end} in Algorithm~\ref{alg:dpar2}) based on the variation of $e = \left(\sum_{k=1}^{K} \|\mat{X}_{k}-\hat{\mat{X}}_{k}\|_{\mathbf{F}}^2\right)$ where $\hat{\mat{X}}_{k} = \mat{Q}_k\mat{H}\mat{S}_{k}\mat{V}^T$ is the $k$th reconstructed slice.
However, measuring reconstruction errors $\sum_{k=1}^{K} \|\mat{X}_{k}-\hat{\mat{X}}_{k}\|_{\mathbf{F}}^2$ is inefficient
%in approaches including \method, which exploit small compressed results instead of input slices.
%This computation
since it requires high time and space costs proportional to input slices $\mat{X}_{k}$.
To efficiently verify the convergence, our idea is to exploit $\mat{A}_{k}\mat{F}^{(k)}\mat{E}\mat{D}^T$ %in $e$
instead of $\mat{X}_{k}$, since the objective of our update process is to minimize the difference between $\mat{A}_{k}\mat{F}^{(k)}\mat{E}\mat{D}^T$ and $\hat{\mat{X}}_{k} = \mat{Q}_k\mat{H}\mat{S}_{k}\mat{V}^T$.
With this idea, we improve the efficiency by computing $\sum_{k=1}^{K} \|\mat{P}_{k}\mat{Z}_{k}^T\mat{F}^{(k)}\mat{E}\mat{D}^T-\mat{H}\mat{S}_{k}\mat{V}^T \|_{\mathbf{F}}^2$, not the reconstruction errors.
Our computation requires the time $\T{O}(JKR^2)$ and space costs $\T{O}(JKR)$ which are much lower than the costs $\T{O}(\sum_{k=1}^{K}{I_kJR})$ and $\T{O}(\sum_{k=1}^{K}{I_kJ})$ of naively computing $\sum_{k=1}^{K} \|\mat{X}_{k}-\hat{\mat{X}}_{k}\|_{\mathbf{F}}^2$, respectively.
%, we exploit $\mat{A}_{k}\mat{F}^{(k)}\mat{E}\mat{D}^T$ in $e$ instead of $\mat{X}_{k}$.
From $\|\mat{P}_{k}\mat{Z}_{k}^T\mat{F}^{(k)}\mat{E}\mat{D}^T-\mat{H}\mat{S}_{k}\mat{V}^T \|_{\mathbf{F}}^2$, we derive $\|\mat{A}_{k}\mat{F}^{(k)}\mat{E}\mat{D}^T-\hat{\mat{X}}_{k} \|_{\mathbf{F}}^2$.
Since the Frobenius norm is unitarily invariant, we modify the computation as follows:
\begin{align*}
%	&	\|\mat{A}_{k}\mat{F}^{(k)}\mat{E}\mat{D}^T-\hat{\mat{X}}_{k} \|_{\mathbf{F}}^2
%	= \|\mat{Q}_{k}^T\mat{A}_{k}\mat{F}^{(k)}\mat{E}-\mat{H}\mat{S}_{k}\mat{V}^T\mat{D} \|_{\mathbf{F}}^2 \\
%	& = \|\mat{P}_{k}\mat{Z}_{k}^T\mat{A}_{k}^T\mat{A}_{k}\mat{F}^{(k)}\mat{E}-\mat{H}\mat{S}_{k}\mat{V}^T\mat{D} \|_{\mathbf{F}}^2 \\
	& \|\mat{P}_{k}\mat{Z}_{k}^T\mat{F}^{(k)}\mat{E}\mat{D}^T-\mat{H}\mat{S}_{k}\mat{V}^T \|_{\mathbf{F}}^2 \\
	& = \|\mat{Q}_k\mat{P}_{k}\mat{Z}_{k}^T\mat{F}^{(k)}\mat{E}\mat{D}^T-\mat{Q}_k\mat{H}\mat{S}_{k}\mat{V}^T \|_{\mathbf{F}}^2 \\ 	
& = \|\mat{A}_{k}\mat{Z}_{k}\mat{P}_{k}^{T}\mat{P}_{k}\mat{Z}_{k}^T\mat{F}^{(k)}\mat{E}\mat{D}^T-\mat{Q}_k\mat{H}\mat{S}_{k}\mat{V}^T \|_{\mathbf{F}}^2 \\	
& = \|\mat{A}_{k}\mat{F}^{(k)}\mat{E}\mat{D}^T-\hat{\mat{X}}_{k} \|_{\mathbf{F}}^2
\end{align*}
where $\mat{P}_{k}^{T}\mat{P}_{k}$ and $\mat{Z}_{k}\mat{Z}_{k}^T$ are equal to $\mat{I} \in \mathbb{R}^{R \times R}$ since $\mat{P}_{k}$ and $\mat{Z}_{k}$ are orthonormal matrices.
Note that the size of $\mat{P}_{k}\mat{Z}_{k}^T\mat{F}^{(k)}\mat{E}\mat{D}^T$ and $\mat{H}\mat{S}_{k}\mat{V}^T$ is $R\times J$ which is much smaller than the size $I_k \times J$ of input slices $\mat{X}_k$.
This modification completes the efficiency of our update rule.

\begin{algorithm} [t]
	\SetNoFillComment
	\caption{Careful distribution of work in \method}
	\label{alg:partitioning}
	\begin{algorithmic} [1]
%		\footnotesize
		\small
		\algsetup{linenosize=\small}

		\renewcommand{\algorithmicrequire}{\textbf{Input:}}
		\renewcommand{\algorithmicensure}{\textbf{Output:}}
		    \REQUIRE the number $T$ of threads, $\mat{X}_{k} \in \mathbb{R}^{I_{k} \times J}$ for $k = 1,...,K$
		    \ENSURE sets $\T{T}_i$ for $i=1,...,T$.
%		\renewcommand{\algorithmicrequire}{\textbf{Parameters:}}
%		\REQUIRE target rank $R$ \\
		\STATE initialize $\T{T}_i \leftarrow \emptyset $ for $i=1,...,T$.
		\STATE construct a list $S$ of size $T$ whose elements are zero
		\STATE construct a list $L_{init}$ containing the number of rows of $\mat{X}_{k}$ for $k=1,...,K$ \label{alg4:line:list_construction} \\
		\STATE sort $L_{init}$ in descending order, and obtain lists $L_{val}$ and $L_{ind}$ that contain sorted values and those corresponding indices \label{alg4:line:sort} \\
		\FOR {$k=1,...,K$}
			\STATE $t_{min} \leftarrow \argmin S$ \label{alg4:line:find_min_set}  \\
			\STATE $l \leftarrow L_{ind}[k]$ \label{alg4:line:find_ind} \\
			\STATE $\T{T}_{t_{min}} \leftarrow \T{T}_{t_{min}} \cup \{\mat{X}_{l} \}$ \label{alg4:line:add_slice} \\
			\STATE $S[t_{min}] \leftarrow S[t_{min}] + L_{val}[k]$ \label{alg4:line:update_S}
		\ENDFOR
	\end{algorithmic}
\end{algorithm}

\vspace{-1mm}
\subsection{Careful distribution of work}
\label{subsec:carefulwork}
The last challenge for an efficient and scalable PARAFAC2 decomposition method is how to parallelize the computations described in Sections~\ref{subsec:compression} to~\ref{subsec:updating_matrices}.
Although a previous work~\cite{PerrosPWVSTS17} introduces the parallelization with respect to the $K$ slices,
there is still a room for maximizing parallelism.
Our main idea is to carefully allocate input slices $\mat{X}_{k}$ to threads by considering the irregularity of a given tensor.
%In addition, \method preserves parallel computations with respect to the $K$ subjects in updating factor matrices.

The most expensive operation is to compute randomized SVD of input slices $\mat{X}_{k}$ for all $k$; thus we first focus on how well we parallelize this computation (i.e., lines~\ref{alg2:line:start_rand} to~\ref{alg2:line:end_rand} in Algorithm~\ref{alg:dpar2}).
A naive approach is to randomly allocate input slices to threads, and let each thread compute randomized SVD of the allocated slices.
However, the completion time of each thread can vary since the computational cost of computing randomized SVD is proportional to the number of rows of slices;
the number of rows of input slices is different from each other as shown in Fig.~\ref{fig:time_length}.
Therefore, we need to distribute $\mat{X}_{k}$ fairly across each thread considering their numbers of rows.

For $i=1,..,T$, consider that an $i$th thread performs randomized SVD for slices in a set $\T{T}_{i}$ where $T$ is the number of threads.
To reduce the completion time, the sums of rows of slices in the sets should be nearly equal to each other.
%To obtain balanced subsets, we allocate slices into subsets by considering the number of rows of slices;
To achieve it, we exploit a greedy number partitioning technique that repeatedly adds a slice into a set with the smallest sum of rows.
Algorithm~\ref{alg:partitioning} describes how to construct the sets $\T{T}_{i}$ for compressing input slices in parallel.
Let $L_{init}$ be a list containing the number of rows of $\mat{X}_k$ for $k=1,...,K$ (line~\ref{alg4:line:list_construction} in Algorithm~\ref{alg:partitioning}).
We first obtain lists $L_{val}$ and $L_{ind}$, sorted values and those corresponding indices, by sorting $L_{init}$ in descending order (line~\ref{alg4:line:sort} in Algorithm~\ref{alg:partitioning}).
We repeatedly add a slice $\mat{X}_{k}$ to a set $\T{T}_{i}$ that has the smallest sum.
For each $k$, we find the index $t_{min}$ of the minimum in $S$ whose $i$th element corresponds to the sum of row sizes of slices in the $i$th set $\T{T}_{i}$ (line~\ref{alg4:line:find_min_set} in Algorithm~\ref{alg:partitioning}).
Then, we add a slice $\mat{X}_{l}$ to the set $\T{T}_{t_{min}}$ where $l$ is equal to $L_{ind}[k]$, and update the list $S$ by $S[t_{min}] \leftarrow S[t_{min}] + L_{val}[k]$ (lines~\ref{alg4:line:find_ind} to~\ref{alg4:line:update_S} in Algorithm~\ref{alg:partitioning}).
Note that $S[k]$, $L_{ind}[k]$, and $L_{val}[k]$ denote the $k$th element of $S$, $L_{ind}$, and $L_{val}$, respectively.
After obtaining the sets $\T{T}_{i}$ for $i=1,..,T$, $i$th thread performs randomized SVD for slices in the set $\T{T}_{i}$.
%alg4:line:sort
%alg4:line:find_min_set
%alg4:line:find_ind
%alg4:line:update_S

%Each thread performs randomized SVD for slices included in its corresponding subset.

After decomposing $\mat{X}_{k}$ for all $k$, we do not need to consider the irregularity for parallelism since there is no computation with $\mat{A}_{k}$ which involves the irregularity.
Therefore, we uniformly allocate computations across threads for all $k$ slices.
%we focus on parallelization with respect to the $K$ subjects.
In each iteration (lines~\ref{alg2:line:start_iter} to~\ref{alg2:line:update_s} in Algorithm~\ref{alg:dpar2}), we easily parallelize computations.
First, we parallelize the iteration (lines~\ref{alg2:line:start_iter} to~\ref{alg2:line:end_fiter}) for all $k$ slices.
%Since the iteration (lines~\ref{alg2:line:start_siter} to~\ref{alg2:line:end_siter} in Algorithm~\ref{alg:dpar2}) is not computed explicitly, we pass.}
% and lines~\ref{alg2:line:start_siter} to~\ref{alg2:line:end_siter} in Algorithm~\ref{alg:dpar2}).
To update $\mat{H}$, $\mat{V}$, and $\mat{W}$, we need to compute $\mat{G}^{(1)}$, $\mat{G}^{(2)}$, and $\mat{G}^{(3)}$ in parallel.
In Lemmas~\ref{lemma:mttkrp_mode1} and~\ref{lemma:mttkrp_mode2}, \method parallelly computes $\mat{W}(k,r)\left(\mat{P}_{k}\mat{Z}_{k}^T\mat{F}^{(k)}\right)$ and $\mat{W}(k,r)\mat{F}^{(k)}\mat{Z}_{k}\mat{P}_{k}^{T}\mat{H}(:,r)$ for $k$, respectively.
In Lemma~\ref{lemma:mttkrp_mode3}, \method parallelly computes $\left(vec\left(\mat{P}_{k}\mat{Z}_{k}^T\mat{F}^{(k)}\right)\right)^T\left(\mat{E}\mat{D}^T\mat{V}(:,r)\otimes \mat{H}(:,r)\right)$ for $k$.
%Note that $\mat{E}\mat{D}^T\mat{V}(:,r)$ are computed once before parallel computations.
\begin{figure}
	\centering
	\vspace{-2mm}
%	 \subfloat{\includegraphics[width=0.5\textwidth]{FIG/PERF_LEGEND.pdf}} \\
%	 \setcounter{subfigure}{0}
	 \subfloat[US stock data]{\includegraphics[width=0.2\textwidth]{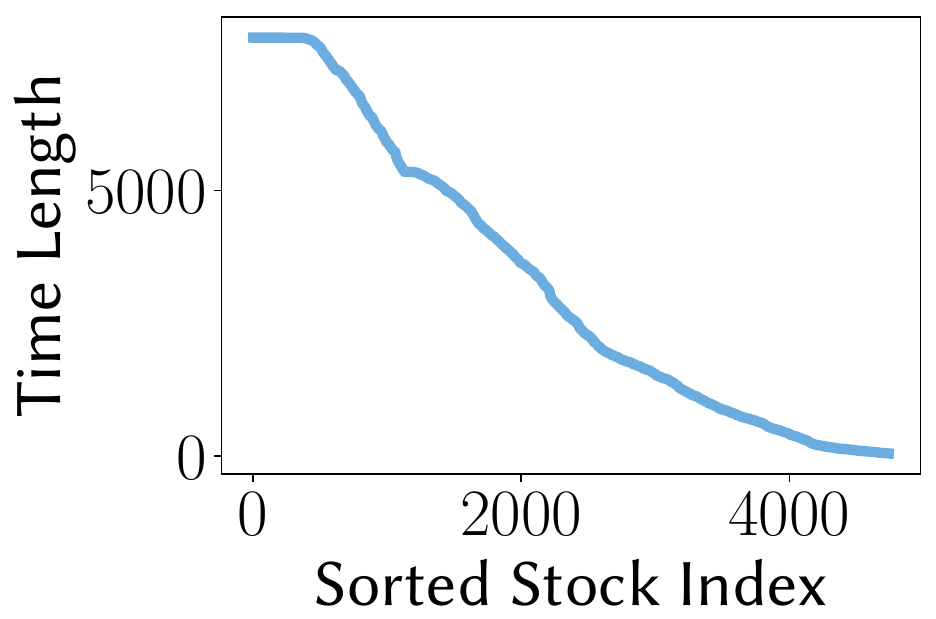}\label{fig:length_US}}
	 \subfloat[KR stock data]{\includegraphics[width=0.2\textwidth]{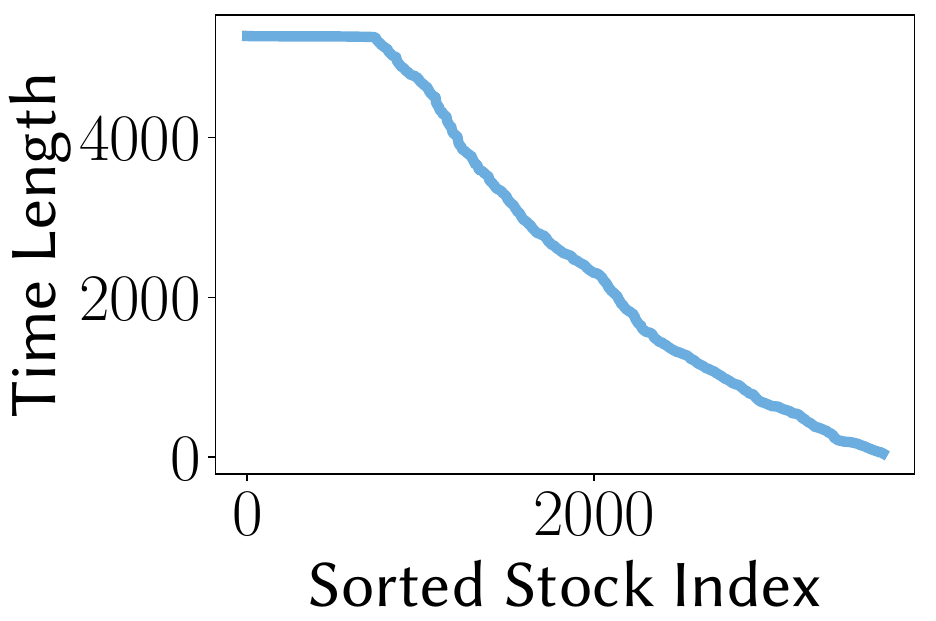}\label{fig:length_KR}}
%	 \subfloat[Fake figure]{\includegraphics[width=0.24\textwidth]{FIG/LENGTH/LENGTH_US.pdf}\label{fig:length_US}}
%	 \subfloat[Fake figure]{\includegraphics[width=0.24\textwidth]{FIG/LENGTH/LENGTH_US.pdf}\label{fig:length_US}}	 	
	 	 \\
	 	\vspace{-1mm}
	\caption{
	The length of temporal dimension of input slices $\mat{X}_{k}$ on US Stock and Korea Stock data.
	We sort the lengths in descending order.
%	\green{See Section~\ref{subsec:exp:querytime} for details.}
	}
	\label{fig:time_length}
\end{figure}

\subsection{Complexities}
\label{subsec:complexities}

We analyze the time complexity of \method.

\begin{lemma}
	\label{lemma:comp}
	Compressing input slices takes $\T{O}\left(\left(\sum_{k=1}^{K}{I_{k}JR}\right)+JKR^2\right)$ time.
\end{lemma}

\begin{proof}
The SVD in the first stage takes $\T{O}\left(\sum_{k=1}^{K}{I_{k}JR}\right)$ times since computing randomized SVD of $\mat{X}_{k}$ takes $\T{O}(I_{k}JR)$ time.
Then, the SVD in the second stage takes $\T{O}\left(JKR^2\right)$ due to randomized SVD of $\mat{M}_{(2)} \in \mathbb{R}^{J \times KR}$.
Therefore, the time complexity of the SVD in the two stages is $\T{O}\left(\left(\sum_{k=1}^{K}{I_{k}JR}\right)+JKR^2\right)$.
\end{proof}

\begin{lemma}
	\label{lemma:update}
	At each iteration, computing $\mat{Y}_{k}$ and updating $\mat{H}$, $\mat{V}$, and $\mat{W}$ takes $\T{O}(JR^2+KR^3)$ time.
\end{lemma}

\begin{proof}
	For $\mat{Y}_{k}$, computing $\mat{F}^{(k)}\mat{E}\mat{D}^T\mat{V}\mat{S}_{k}\mat{H}^{T}$ and performing SVD of it for all $k$ take $\T{O}(JR^2 + KR^3)$.
	Updating each of $\mat{H}$, $\mat{V}$, and $\mat{W}$ takes $\T{O}(JR^2+KR^3+R^3)$ time.
%	 while updating $\mat{W}$ takes $\T{O}(JR^2+KR^4+R^3)$.
	Therefore, the complexity for $\mat{Y}_{k}$, $\mat{H}$, $\mat{V}$, and $\mat{W}$ is $\T{O}\left(JR^2+KR^3\right)$.
\end{proof}

\begin{theorem}
The time complexity of
	\method is $\T{O}\left(\left(\sum_{k=1}^{K}{I_{k}JR}\right)+JKR^2+ MKR^3\right)$ where $M$ is the number of iterations.
\end{theorem}

\begin{proof}
	The overall time complexity of \method is the summation of the compression cost (see Lemma~\ref{lemma:comp}) and the iteration cost (see Lemma~\ref{lemma:update}): $\T{O}\left(\left(\sum_{k=1}^{K}{I_{k}JR}\right)+JKR^2 + M(JR^2 + KR^3)\right)$.
	Note that $MJR^2$ term is omitted since it is much smaller than $\left(\sum_{k=1}^{K}{I_{k}JR}\right)$ and $JKR^2$.
\end{proof}

\begin{theorem}
	The size of preprocessed data of \method is $\T{O}\left(\left(\sum_{k=1}^{K}{I_k R}\right) + KR^2 + JR\right)$.
\end{theorem}

\begin{proof}
	The size of preprocessed data of \method is proportional to the size of $\mat{E}$, $\mat{D}$, $\mat{A}_{k}$, and $\mat{F}^{(k)}$ for $k=1,...,K$.
	The size of $\mat{E}$ and $\mat{D}$ is $R$ and $J\times R$, respectively.
	For each $k$, the size of $\mat{A}$ and $\mat{F}$ is $I_k \times R$ and $R \times R$, respectively.
	Therefore, the size of preprocessed data of \method is $\T{O}\left(\left(\sum_{k=1}^{K}{I_k R}\right) + KR^2 + JR\right)$.
\end{proof}

\section{Experiments}
\label{sec:experiment}

\begin{table}[t!]
%	\vspace{-2mm}
	\caption{Description of real-world tensor datasets.
%	in $\mathbb{R}^{t \times c}$ where $t$ corresponds to total length (time), and $c$ corresponds to the number of time series.
	}
	\centering
	\label{tab:Description}
		\resizebox{0.49\textwidth}{!}{
	\begin{tabular}{lrrrrr}
		\toprule
		\textbf{Dataset} & \textbf{Max Dim. $I_k$} & \textbf{Dim. $J$} & \textbf{Dim. $K$}  & \textbf{Summary} \\
		\midrule
%		 Brainq\footnoteref{foot:brainq}~\cite{MitchellT2008Predicting}  & $3$ & $(360, 21764, 9)$ & $(10, 10, 5)$ \\
		FMA\footnoteref{foot:fma}~\cite{fma_dataset} &  $704$ & $2,049$ & $7,997$ & music \\
		Urban\footnoteref{foot:urban}~\cite{urban} &  $174$ & $2,049$ & $8,455$ & urban sound  \\		
		US Stock\footnoteref{foot:us} &  $7,883$ & $88$ & $4,742$ & stock \\
		Korea Stock\footnoteref{foot:korea}~\cite{jang2021fast} &  $5,270$ & $88$ & $3,664$ & stock\\
		Activity\footnoteref{foot:activity}~\cite{WangLWY12,Karim2018} &  $553$ & $570$ & $320$ & video feature \\
		Action\footnoteref{foot:activity}~\cite{WangLWY12,Karim2018} &  $936$ & $570$ & $567$ & video feature\\
		Traffic\footnoteref{foot:traffic}~\cite{schimbinschi2015traffic} &  $2,033$ & $96$ & $1,084$ & traffic
		 \\		
		PEMS-SF\footnoteref{foot:pems} &  $963$ & $144$ & $440$ & traffic
		 \\				
		\bottomrule
	\end{tabular}}
\end{table}

\begin{figure*}
	\centering
	\vspace{-3mm}
	 \subfloat{\includegraphics[width=0.5\textwidth]{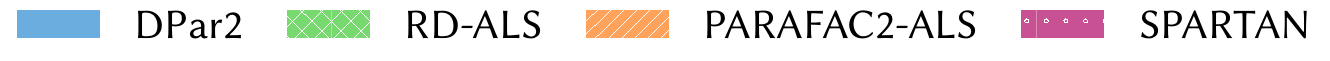}} \\
	\vspace{-3mm}	
	 \setcounter{subfigure}{0}
	 \subfloat[Preprocessing time]{\includegraphics[width=0.42\textwidth]{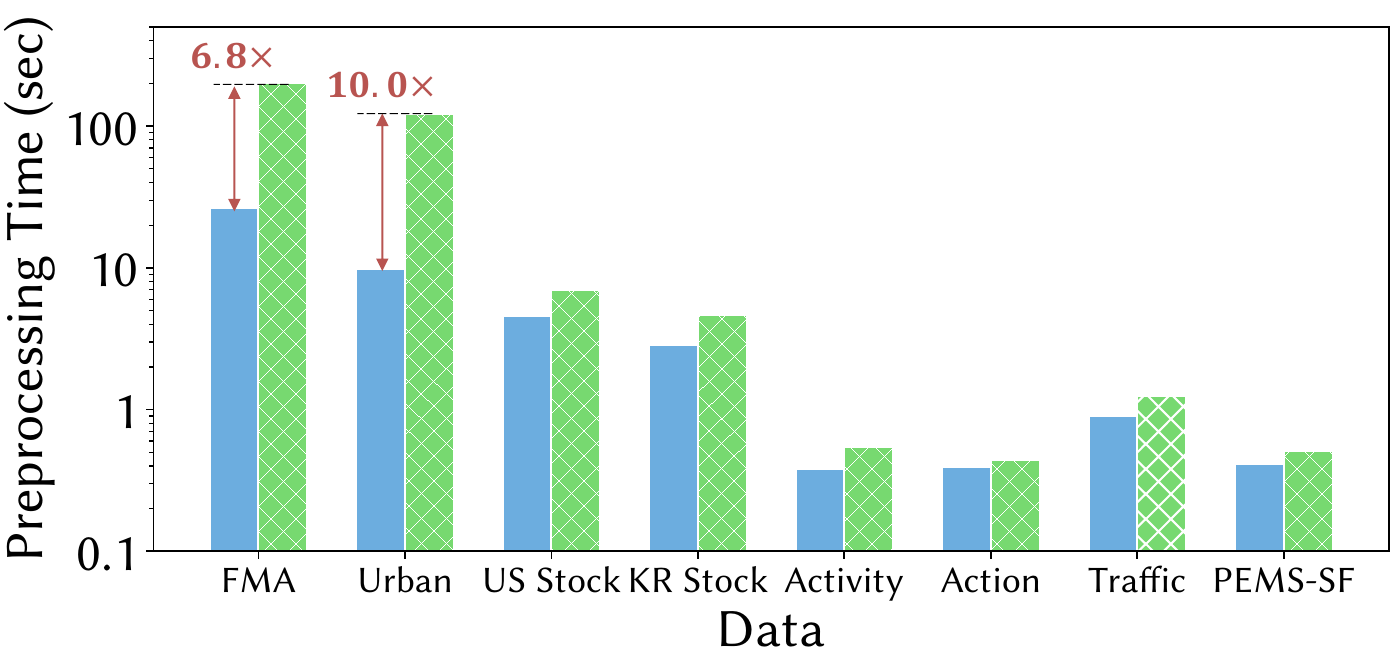}\label{fig:prep_time}}
	 \subfloat[Iteration time]{\includegraphics[width=0.42\textwidth]{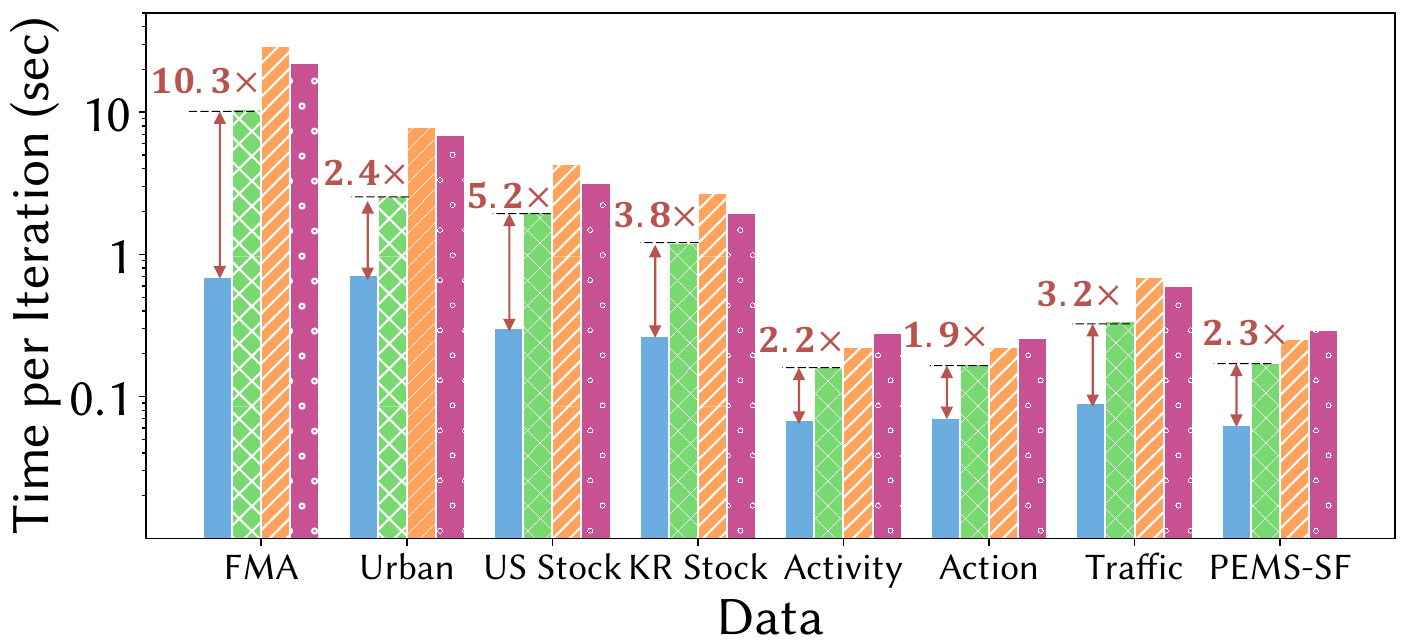}\label{fig:iter_time}}	 	
%	 \subfloat[Preprocessed size]{\includegraphics[width=0.32\textwidth]{FIG/PERF/MEMORY.pdf}\label{fig:prep_size}}		
%	 \qquad
%	\includegraphics[width=0.3\textwidth]{FIG/PERF/MEMORY.pdf}
%	 \subfloat[The size of preprocessed data]{\includegraphics[width=0.32\textwidth]{FIG/PERF/MEMORY.pdf}\label{fig:preprocessed_size}}
%	 \subfloat[Scalability for dimensionality]{\includegraphics[width=0.24\textwidth]{FIG/PERF/SCALABILITY_DIMEN.pdf}\label{fig:scala_dimen}}
%	 \subfloat[Scalability for rank]{\includegraphics[width=0.24\textwidth]{FIG/PERF/SCALABILITY_RANK.pdf}\label{fig:scala_rank}}	
%	 \subfloat[KR Stock]{\includegraphics[width=0.24\textwidth]{FIG/PERF/PERF_STOCK_KR.pdf}\label{fig:perf_stock_kr}}	 	 	
	 	 \\
	 	\vspace{-1mm}
	\caption{\textbf{[Best viewed in color]}
	(a) \method efficiently preprocesses a given irregular dense tensor, which is up to $10\times$ faster compared to \rdals.
	(b) At each iteration, \method runs by up to $10.3\times$ faster than the second best method.
%Thanks to the small preprocessed data, \method has low memory requirements at each iteration.
%	\green{See Section~\ref{subsec:exp:querytime} for details.}
	}
	\label{fig:performance_bar}
\end{figure*}

In this section, we experimentally evaluate the performance of \method.
We answer the following questions:
\begin{itemize*}
	\item[Q1] \textbf{Performance (Section~\ref{subsec:perf}).} How quickly and accurately does \method perform PARAFAC2 decomposition compared to other methods?
	\item[Q2] \textbf{Data Scalability (Section~\ref{subsec:data_scalability}).} How well does \method scale up with respect to tensor size and target rank?
	\item[Q3] \textbf{Multi-core Scalability (Section~\ref{subsec:machine_scalability}).} How much does the number of threads affect the running time of \method?	
	\item[Q4] \textbf{Discovery (Section~\ref{subsec:discovery}).} What can we discover from real-world tensors using \method?
\end{itemize*}

\subsection{Experimental Settings}
\label{subsec:setting}
We describe experimental settings for the datasets, competitors, parameters, and environments.

\textbf{Machine.}
We use a workstation with $2$ CPUs (Intel Xeon E5-2630 v4 @ 2.2GHz), each of which has $10$ cores, and 512GB memory for the experiments.

%\textbf{Dataset.}

\textbf{Real-world Data.}
We evaluate the performance of \method and competitors on real-world datasets summarized in Table~\ref{tab:Description}.
FMA dataset\footnote{\url{https://github.com/mdeff/fma}\label{foot:fma}}~\cite{fma_dataset} is the collection of songs.
Urban Sound dataset\footnote{\url{https://urbansounddataset.weebly.com/urbansound8k.html}\label{foot:urban}}~\cite{urban} is the collection of urban sounds such as drilling, siren, and street music.
For the two datasets, we convert each time series into an image of a log-power spectrogram so that their forms are (time, frequency, song; value) and (time, frequency, sound; value), respectively.
US Stock dataset\footnote{\url{https://datalab.snu.ac.kr/dpar2}\label{foot:us}} is the collection of stocks on the US stock market.
Korea Stock dataset\footnote{\url{https://github.com/jungijang/KoreaStockData}\label{foot:korea}}~\cite{jang2021fast} is the collection of stocks on the South Korea stock market.
Each stock is represented as a matrix of (date, feature) where the feature dimension includes 5 basic features and 83 technical indicators.
The basic features collected daily are the opening, the closing, the highest, and the lowest prices and trading volume, and technical indicators are calculated based on the basic features.
The two stock datasets have the form of (time, feature, stock; value).
%We use the Python library\footnote{\url{https://github.com/bukosabino/ta}} for the technical indicators.
Activity data\footnote{\url{https://github.com/titu1994/MLSTM-FCN}\label{foot:activity}} and Action data\footnoteref{foot:activity} are the collection of features for motion videos.
The two datasets have the form of (frame, feature, video; value).
We refer the reader to~\cite{WangLWY12} for their feature extraction.
Traffic data\footnote{\url{https://github.com/florinsch/BigTrafficData}\label{foot:traffic}} is the collection of traffic volume around Melbourne, and its form is (sensor, frequency, time; measurement).
PEMS-SF data\footnote{\url{http://www.timeseriesclassification.com/}\label{foot:pems}} contain the occupancy rate of different car lanes of San Francisco bay area freeways: (station, timestamp, day; measurement).
Traffic data and PEMS-SF data are 3-order regular tensors, but we can analyze them using PARAFAC2 decomposition approaches.

\begin{figure}
\centering
	\includegraphics[width=0.42\textwidth]{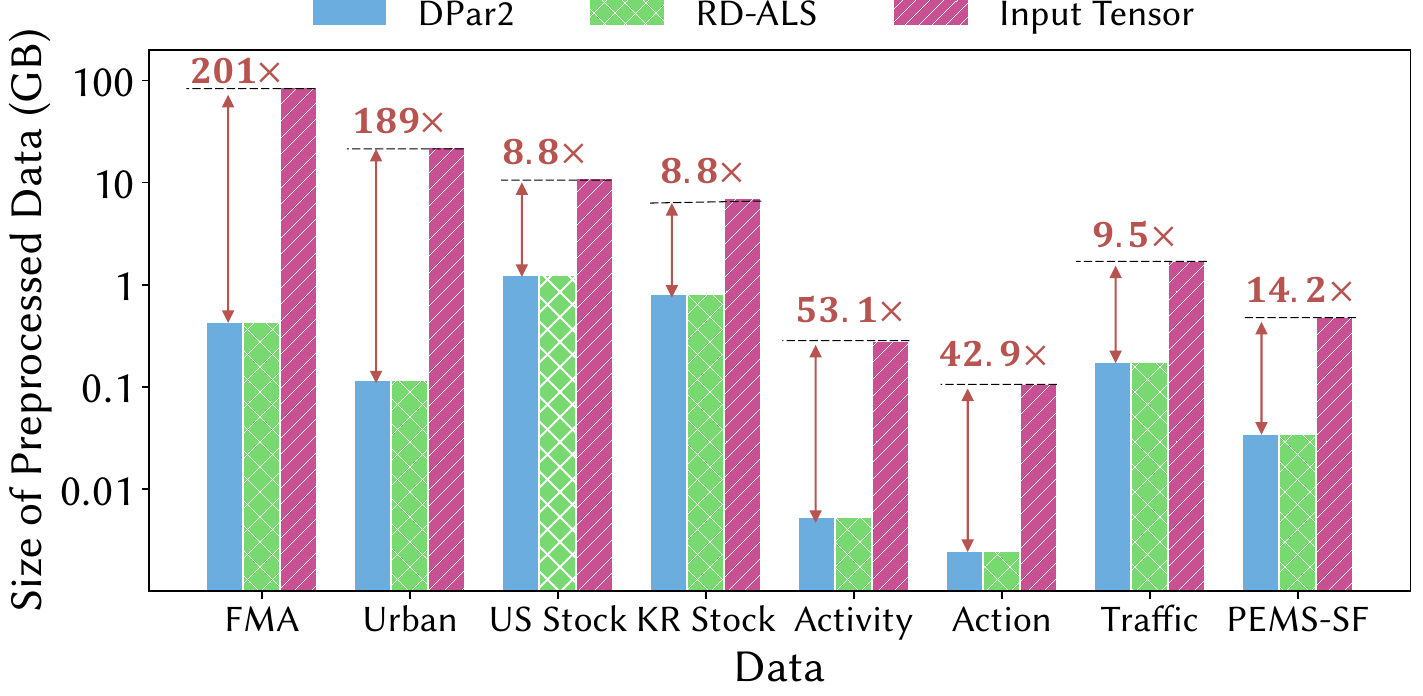}
	\caption{The size of preprocessed data. \method generates up to $201\times$ smaller preprocessed data than input tensors used for \spartan and \als.}
\label{fig:preprocessed_size}
\end{figure}

\textbf{Synthetic Data.}
We evaluate the scalability of \method and competitors on synthetic tensors.
Given the number $K$ of slices, and the slice sizes $I$ and $J$, we generate a synthetic tensor using $\textit{tenrand(I, J, K)}$ function in Tensor Toolbox~\cite{TTB_Software}, which randomly generates a tensor $\T{X} \in \mathbb{R}^{I\times J \times K}$.
We construct a tensor $\{ \mat{X}_k \}_{k=1}^{K}$ where $\mat{X}_{k}$ is equal to $\T{X}(:,:,k)$ for $k=1,...K$.
% = \mat{U}_{k}\mat{V}_{k}^T + \eta \mat{N}_{k}$ for $k=1,...,K$ where matrices $\mat{U}_{k} \in \mathbb{R}^{I\times R}$, $\mat{V}_{k} \in \mathbb{R}^{J\times R}$, and $\mat{N}_{k} \in \mathbb{R}^{I\times J}$ are randomly generated, and $\eta$ is noise level.
%We set $\eta$ to $0.01$.

\textbf{Competitors.}
We compare \method with PARAFAC2 decomposition methods based on ALS.
All the methods including \method are implemented in MATLAB (R2020b).

\begin{itemize}[noitemsep,topsep=0pt]
	\item \textbf{\method}: the proposed PARAFAC2 decomposition model which preprocesses a given irregular dense tensor and updates factor matrices using the preprocessing result.
%	\item \textbf{\method with a higher rank}: \method with a higher rank. We further use an additional target rank of about $3-5$ depending on the data.
	\item \textbf{\rdals~\cite{ChengH19}}: PARAFAC2 decomposition which preprocesses a given irregular tensor. Since there is no public code, we implement it using Tensor Toolbox~\cite{TTB_Software} based on its paper~\cite{ChengH19}.
	\item \textbf{\als}: PARAFAC2 decomposition based on ALS approach. It is implemented based on Algorithm~\ref{alg:als} using Tensor Toolbox~\cite{TTB_Software}.
	\item \textbf{\spartan~\cite{PerrosPWVSTS17}}: fast and scalable PARAFAC2 decomposition for irregular sparse tensors. Although it targets on sparse irregular tensors, it can be adapted to irregular dense tensors. We use the code implemented by authors\footnote{\url{https://github.com/kperros/SPARTan}}.
\end{itemize}

\begin{figure*}
%	\vspace{-3mm}
\qquad\qquad\raggedright
	 \subfloat{\includegraphics[width=0.55\textwidth]{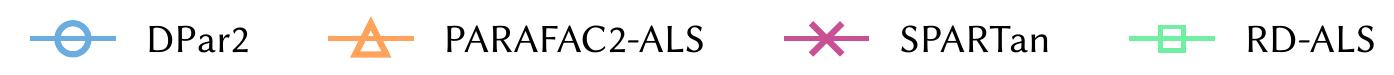}} \\
	\vspace{-3mm}	
	 \setcounter{subfigure}{0}
	\centering	
	 \subfloat[Scalability for tensor size]{\includegraphics[width=0.28\textwidth]{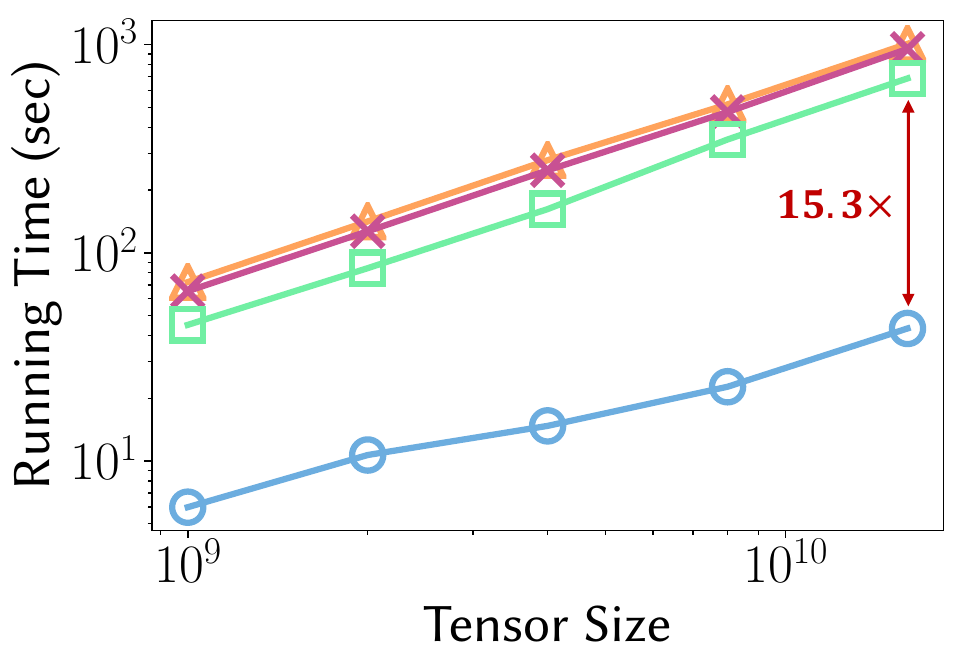}\label{fig:scala_dimen}}
	 \subfloat[Scalability for rank]{\includegraphics[width=0.28\textwidth]{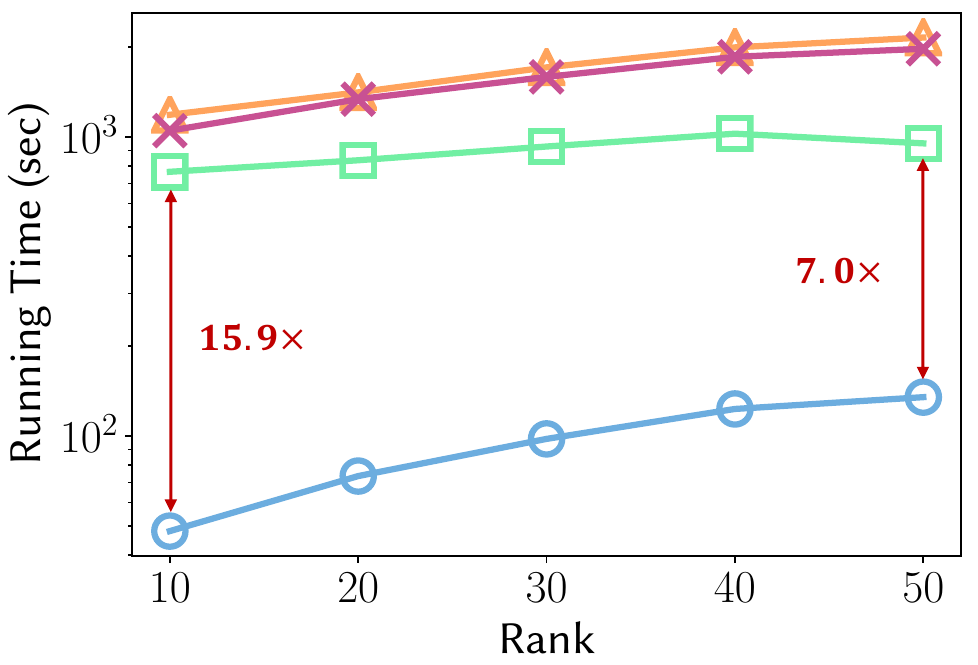}\label{fig:scala_rank}}	
	 	\subfloat[Machine Scalability]{\includegraphics[width=0.29\textwidth]{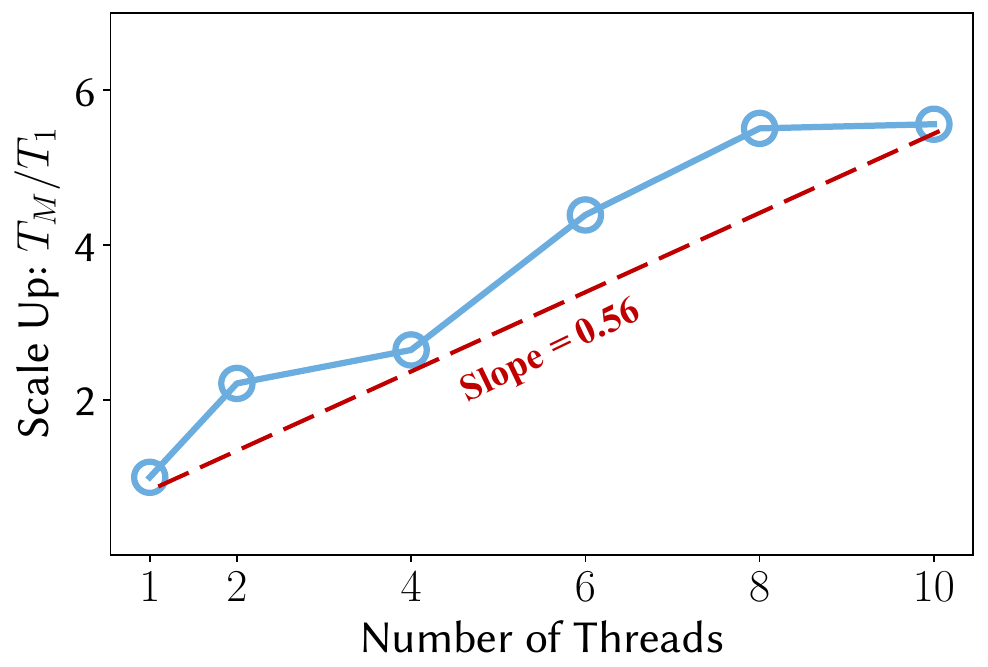}\label{fig:scalability_threads}}
%	 \subfloat[KR Stock]{\includegraphics[width=0.24\textwidth]{FIG/PERF/PERF_STOCK_KR.pdf}\label{fig:perf_stock_kr}}	 	 	
	 	 \\
%	 	\vspace{-1mm}
	\caption{Data scalability.
	\method is more scalable than other PARAFAC2 decomposition methods in terms of both tensor size and rank.
	(a) \method is $15.3\times$ faster than the second-fastest method on the irregular dense tensor of the total size $1.6\times 10^{10}$.
	(b) \method is $7.0\times$ faster than the second-fastest method even when a high target rank is given.
	(c) Multi-core scalability with respect to the number of threads. $T_M$ indicates the running time of \method on the number $M$ of threads. \method gives near-linear scalability, and accelerates $5.5\times$ when the number of threads increases from $1$ to $10$.
%	\green{See Section~\ref{subsec:exp:querytime} for details.}
	}
	\label{fig:scalability}
\end{figure*}

\textbf{Parameters.}
{
We use the following parameters.
\begin{itemize}[noitemsep,topsep=0pt]
	\item \textbf{Number of threads:} we use $6$ threads except in Section~\ref{subsec:machine_scalability}. % using \textit{maxNumCompThreads(1)} function in MATLAB.
	\item \textbf{Max number of iterations:} the maximum number of iterations is set to $32$.
	\item \textbf{Rank:} we set the target rank $R$ to $10$ except in the trade-off experiments of Section~\ref{subsec:perf} and Section~\ref{subsec:machine_scalability}.
We also set the rank of randomized SVD to $10$ which is the same as the target rank $R$ of PARAFAC2 decomposition.
%	\item \textbf{Tolerance:}
\end{itemize}
%{Other parameters for competitors are set to the values proposed in each paper.}
To compare running time, we run each method $5$ times, and report the average.
}

\textbf{Fitness.}
We evaluate the fitness defined as follows:
\begin{align*}
1 - \left(\frac{\sum_{k=1}^{K} \|\mat{X}_{k}-\hat{\mat{X}}_{k}\|_{\mathbf{F}}^2}{\sum_{k=1}^{K} \|\mat{X}_{k}\|_{\mathbf{F}}^2}\right)
\end{align*}
where $\mat{X}_{k}$ is the $k$-th input slice and $\hat{\mat{X}}_{k}$ is the $k$-th reconstructed slice of PARAFAC2 decomposition.
Fitness close to $1$ indicates that a model approximates a given input tensor well.

\subsection{Performance (Q1)}
\label{subsec:perf}

We evaluate the fitness and the running time of \method, \rdals, \spartan, and \als.

\textbf{Trade-off.}
Fig.~\ref{fig:performance} shows that \method provides the best trade-off of running time and fitness on real-world irregular tensors for the three target ranks: $10$, $15$, and $20$.
\method achieves $6.0\times$ faster running time than the competitors for FMA dataset while having a comparable fitness.
In addition, \method provides at least $1.5\times$ faster running times than the competitors for the other datasets.
The performance gap is large for FMA and Urban datasets whose sizes are larger than those of the other datasets.
It implies that \method is more scalable than the competitors in terms of tensor sizes.
%
%When \method uses a higher rank than the competitors, \method achieves faster and higher fitness than the competitors, simultaneously.}
%\method with a higher target rank $R+\alpha$ (e.g., $\alpha = 3$ to $5$) achieves up to $9.1\times$ faster than the competitors for FMA dataset while having higher fitness.
%In addition, \method with a higher target rank provides at least $2.0$ faster than the competitors for the other datasets.
%Although \method with the target rank $R$ has slightly lower fitness than competitors, it is much faster than them.
%This is because our proposed method minimizes the overheads to handle large-scale tensors including FMA and Urban datasets.
%In addition, \method with an additional rank (about $3-5$) outperforms the competitors in terms of both speed and fitness.

\textbf{Preprocessing time.}
We compare \method with \rdals and exclude \spartan and \als since only \rdals has a preprocessing step.
As shown in Fig.~\ref{fig:prep_time}, \method is up to $10\times$ faster than \rdals.
There is a large performance gap on FMA and Urban datasets since \rdals cannot avoid the overheads for the large tensors.
\rdals performs SVD of the concatenated slice matrices $\concat_{k=1}^{K}{\mat{X}^T_k}$, which leads to its slow preprocessing time. %making it slow.
%The preprocessing time of \rdals is comparable to that of \method for all datasets except for the two datasets.

\textbf{Iteration time.}
Fig.~\ref{fig:iter_time} shows that \method outperforms competitors for running time at each iteration.
Compared to \spartan and \als, \method significantly reduces the running time per iteration due to the small size of the preprocessed results.
Although \rdals reduces the computational cost at each iteration by preprocessing a given tensor, \method is up to $10.3\times$ faster than \rdals.
Compared to \rdals that computes the variation of $\left(\sum_{k=1}^{K} \|\mat{X}_{k}-\mat{Q}_k\mat{H}\mat{S}_{k}\mat{V}^T\|_{\mathbf{F}}^2\right)$ for the convergence criterion, \method efficiently verifies the convergence by computing the variation of $\sum_{k=1}^{K} \|\mat{P}_{k}\mat{Z}_{k}^T\mat{F}^{(k)}\mat{E}\mat{D}^T-\mat{H}\mat{S}_{k}\mat{V}^T \|_{\mathbf{F}}^2$,
which affects the running time at each iteration.
In summary, \method obtains $\mat{U}_{k}$, $\mat{S}_{k}$, and $\mat{V}$ in a reasonable running time even if the number of iterations increases.

\textbf{Size of preprocessed data.}
We measure the size of preprocessed data on real-world datasets.
For \als and \spartan, we report the size of input irregular tensor since they have no preprocessing step.
Compared to an input irregular tensor, \method generates much smaller preprocessed data by up to $201$ times as shown in Fig.~\ref{fig:preprocessed_size}.
Given input slices $\mat{X}_k$ of size $I_k \times J$, the compression ratio increases as the number $J$ of columns increases;
the compression ratio is larger on FMA, Urban, Activity, and Action datasets than on US Stock, KR Stock, Traffic, and PEMS-SF.
This is because the compression ratio is proportional to $\frac{\text{Size of an irregular tensor}}{\text{Size of the preprocessed results}} \approx \frac{IJK}{IKR + KR^2 + JR} = \frac{1}{R/J + R^2/IJ + R/IK}$ assuming $I_1 = ... = I_K = I$; $R/J$ is the dominant term which is much larger than $R^2/IJ$ and $R/IK$.
%We assume that $I_1 = ... = I_K = I$.

\subsection{Data Scalability (Q2)}
\label{subsec:data_scalability}
We evaluate the data scalability of \method by measuring the running time on several synthetic datasets.
We first compare the performance of \method and the competitors by increasing the size of an irregular tensor.
Then, we measure the running time by changing a target rank.

\textbf{Tensor Size.}
To evaluate the scalability with respect to the tensor size, we generate $5$ synthetic tensors of the following sizes $I\times J \times K$: $\{ 1000\times 1000\times 1000, 1000\times 1000\times 2000, 2000\times 1000\times 2000, 2000\times 2000\times 2000, 2000\times 2000\times 4000  \}$.
For simplicity, we set $I_1= \cdots = I_K = I$.
%\blue{We set the target rank $R$ to $10$ for all methods except for \method with a higher rank that uses the rank $R+5$.}
Fig.~\ref{fig:scala_dimen} shows that \method is up to $15.3\times$ faster than competitors on all synthetic tensors; in addition, the slope of \method is lower than that of competitors.
Also note that only \method obtains factor matrices of PARAFAC2 decomposition within a minute for all the datasets.

\textbf{Rank.}
To evaluate the scalability with respect to rank, we generate the following synthetic data: $I_1= \cdots = I_K = 2,000$, $J = 2,000$, and $K = 4,000$.
Given the synthetic tensors, we measure the running time for $5$ target ranks: $10$, $20$, $30$, $40$, and $50$.
%As similar to the tensor size,
\method is up to $15.9\times$ faster than the second-fastest method with respect to rank in Fig.~\ref{fig:scala_rank}.
For higher ranks, the performance gap slightly decreases since \method depends on the performance of randomized SVD which is designed for a low target rank.
Still, \method is up to $7.0\times$ faster than competitors with respect to the highest rank used in our experiment.

%\textbf{Data not fit in memory.}
\begin{figure*}
	\centering
	\vspace{-4mm}
%	 \subfloat{\includegraphics[width=0.5\textwidth]{FIG/PERF_LEGEND.pdf}} \\
%	 \setcounter{subfigure}{0}
	 \subfloat[US stock data]{\includegraphics[width=0.4\textwidth]{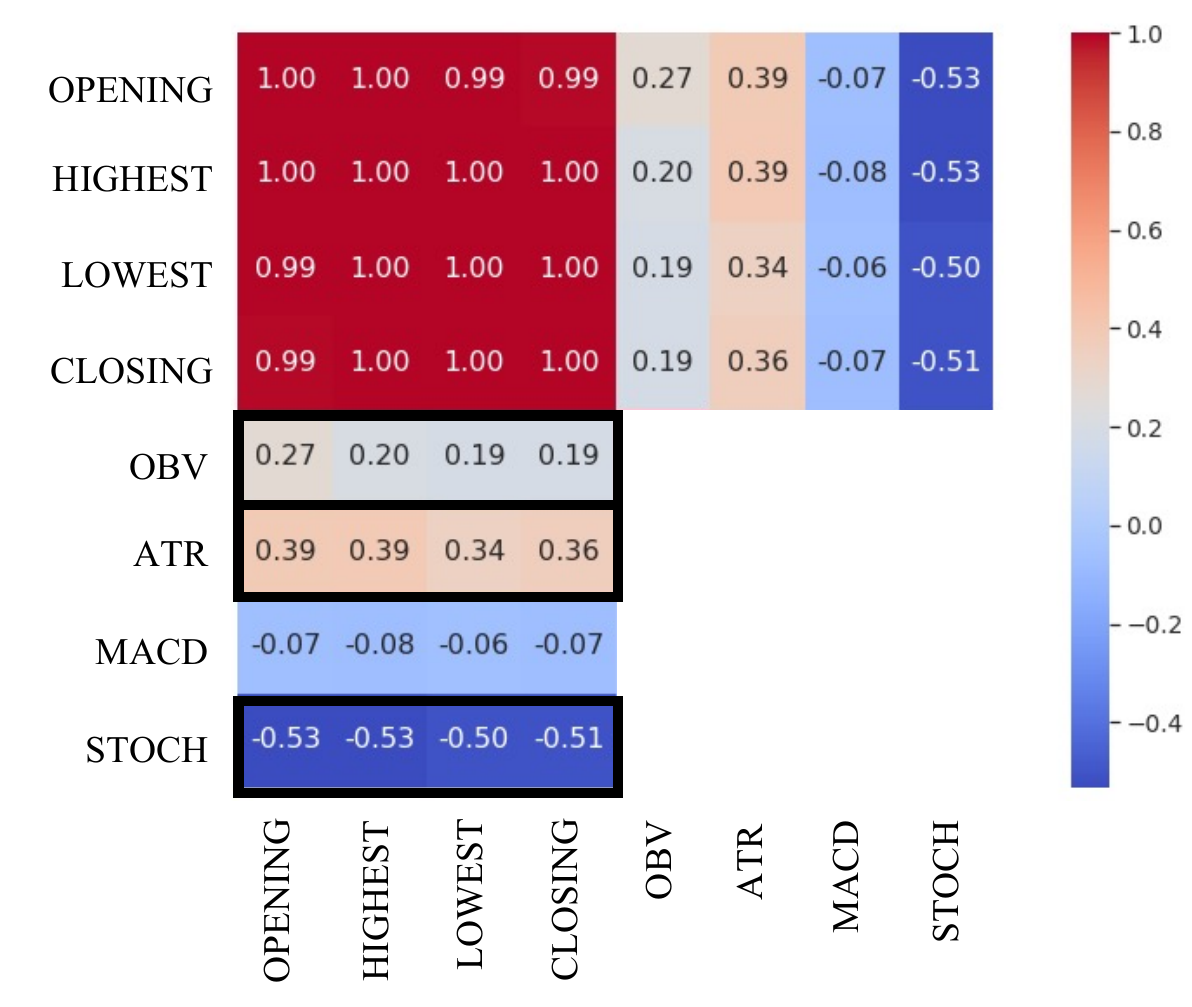}\label{fig:V_sim_us}}
	 \subfloat[Korea stock data]{\includegraphics[width=0.4\textwidth]{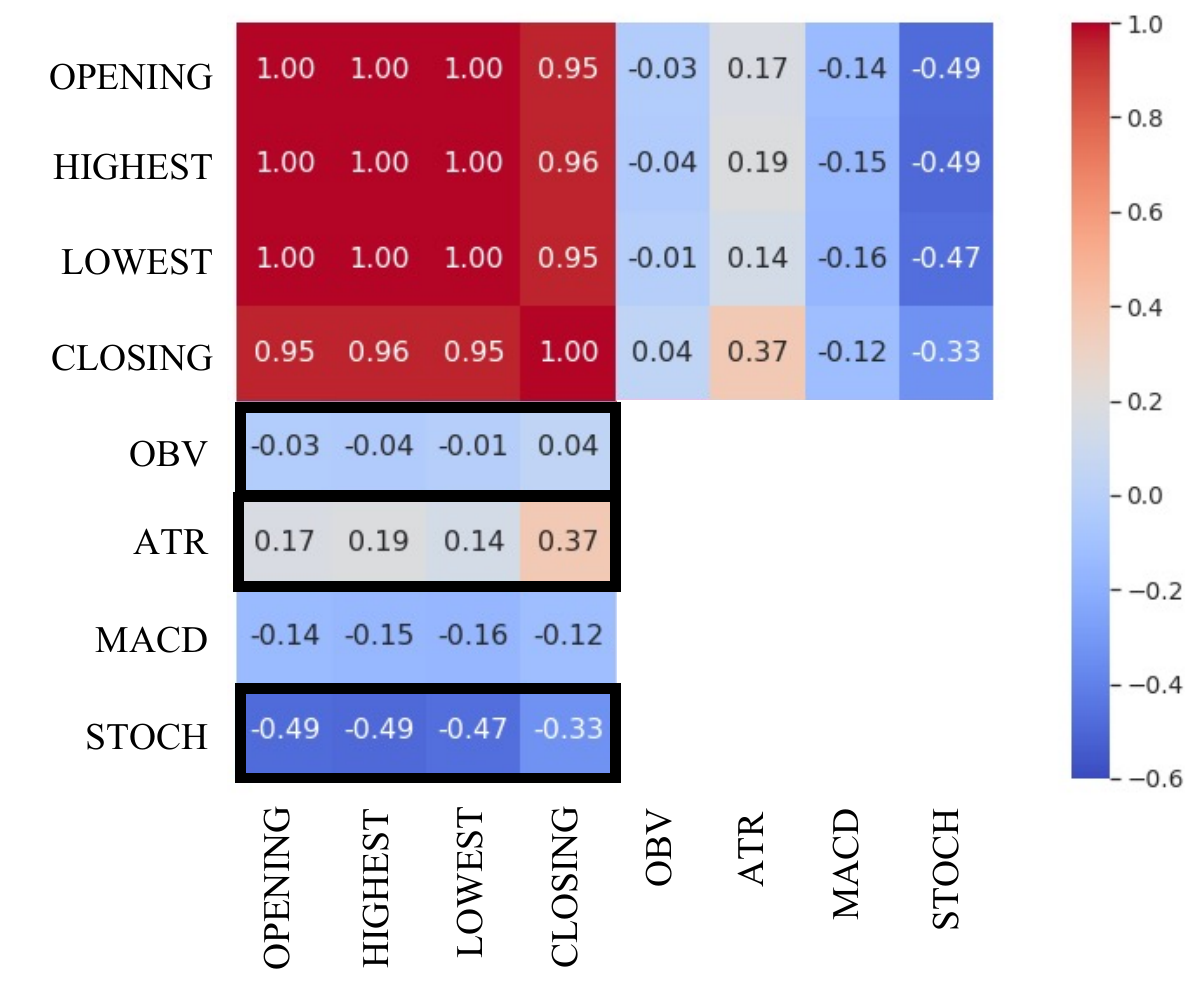}\label{fig:V_sim_kr}}	
%	 \subfloat[KR stock data]{\includegraphics[width=0.235\textwidth]{FIG/LENGTH/LENGTH_KR.pdf}\label{fig:length_KR}}
%	 \subfloat[Fake figure]{\includegraphics[width=0.24\textwidth]{FIG/LENGTH/LENGTH_US.pdf}\label{fig:length_US}}
%	 \subfloat[Fake figure]{\includegraphics[width=0.24\textwidth]{FIG/LENGTH/LENGTH_US.pdf}\label{fig:length_US}}	 	
	 	 \\
	 	\vspace{-1mm}
	\caption{
The similarity patterns of features are different on the two stock markets.
%Similarity of features in US and South Korea Stock datasets.
(a) For US Stock data, ATR and OBV have a positive correlation with the price features.
(b) For Korea Stock data, they are uncorrelated with the price features in general.
	}
	\label{fig:V_sim}
\end{figure*}

\subsection{Multi-core Scalability (Q3)}
\label{subsec:machine_scalability}
We generate the following synthetic data: $I_1= \cdots = I_K = 2,000$, $J = 2,000$, and $K = 4,000$, and evaluate the multi-core scalability of \method with respect to the number of threads: $1,2,4,6,8,$ and $10$.
$T_M$ indicates the running time when using the number $M$ of threads. % : $T_1$ is the running time for one thread.
As shown in Fig.~\ref{fig:scalability_threads},
\method gives near-linear scalability, and
accelerates $5.5\times$ when the number of threads increases from $1$ to $10$.

\subsection{Discoveries (Q4)}
\label{subsec:discovery}
%\method enables analysis from various perspectives.
%on the two stock datasets: US Stock and Korea Stock datasets.
We discover various patterns using \method on real-world datasets.
%Given an irregular dense tensor, we decompose it into factor matrices using \method.
%Then, we analyze the results with post-processing.

\subsubsection{\textbf{Feature Similarity on Stock Dataset}}
\label{subsec:discovery:feature_similarity}

We measure the similarities between features on US Stock and Korea Stock datasets, and compare the results.
We compute Pearson Correlation Coefficient (PCC) between $\mat{V}(i,:)$, which represents a latent vector of the $i$th feature.
For effective visualization, we select 4 price features (the opening, the closing, the highest, and the lowest prices), and 4 representative technical indicators described as follows:
\begin{itemize}[noitemsep,topsep=0pt]
	\item \textbf{OBV (On Balance Volume)}: a technical indicator for cumulative trading volume. If today's closing price is higher than yesterday's price, OBV increases by the amount of today's volume. If not, OBV decreases by the amount of today's volume.
	\item \textbf{ATR (Average True Range)}: a technical indicator for volatility developed by J. Welles Wilder, Jr.
			It increases in high volatility while decreasing in low volatility.
	\item \textbf{MACD (Moving Average Convergence and Divergence)}: a technical indicator for trend developed by Gerald Appel. It indicates the difference between long-term and short-term exponential moving averages (EMA).			
	\item \textbf{STOCH (Stochastic Oscillator)}: a technical indicator for momentum developed by George Lane.
			It indicates the position of the current closing price compared to the highest and the lowest prices in a duration.
\end{itemize}
%Each indicator represents the following categories, respectively: volatility, momentum, volume, and trend.

%Fig.~\ref{fig:V_sim} shows , and
Fig.~\ref{fig:V_sim_us} and~\ref{fig:V_sim_kr} show correlation heatmaps for US Stock data and Korea Stock data, respectively.
We analyze correlation patterns between price features and technical indicators.
On both datasets, STOCH has a negative correlation and MACD has a weak correlation with the price features.
On the other hand, OBV and ATR indicators have different patterns on the two datasets.
On the US stock dataset, ATR and OBV have a positive correlation with the price features.
On the Korea stock dataset, OBV has little correlation with the price features.
Also, ATR has little correlation with the price features except for the closing price.
These different patterns are due to the difference of the two markets in terms of market size, market stability, tax, investment behavior, etc.
%There are different correlation patterns on the two datasets.
%On the US stock market, STOCH has a negative correlation with the price features while the other indicators have little correlation with the price features.
%The Korean stock market has a different tendency;
%ATR has a positive correlation with the price features, but MACD has little correlation with them, in contrast to the US stock market.
%These differences are from % including market size, tax, investors' preference, stability, etc.
%
%In Fig.~\ref{fig:V_sim_us}, the three features, ATR, STOCH, and OBV, have a positive correlation with each other.
%In addition, MACD has a negative correlation with the price features.
%In Fig.~\ref{fig:V_sim_kr}, the patterns are different:
%the four technical indicators have little correlation with each other.
%OBV has a positive correlation with the price features.
%The similarity patterns of features are different on the two stock markets.
%The correlations for other technical indicators can be also analyzed based on the factor matrix $\mat{V}$.

\subsubsection{\textbf{Finding Similar Stocks}}
\label{subsec:discovery:similarity}

\textit{On US Stock dataset, which stock is similar to a target stock $s_T$ in a time range that a user is curious about?}
In this section, we provide analysis by setting the target stock $s_T$ to \textit{Microsoft} (Ticker: MSFT), and the range a duration when the COVID-19 was very active (Jan. 2, 2020 - Apr. 15, 2021).
We efficiently answer the question by 1) constructing the tensor included in the range, 2) obtaining factor matrices with \method, and 3) post-processing the factor matrices of \method.
Since $\mat{U}_{k}$ represents temporal latent vectors of the $k$th stock,
%Based on the result of \method on US Stock and KR Stock datasets,
%we find similar stocks of our target stocks, Apple and Samsung Electronics, respectively.
%They are the top market capitalization stocks in US and Korea stock markets, respectively.
the similarity $sim(s_i, s_j)$ between stocks $s_i$ and $s_j$ is computed as follows:
\begin{align}
\label{eq:similarity}
	sim(s_i, s_j) = \exp\left(-\gamma{{\|\mat{U}_{s_i} - \mat{U}_{s_j}\|_F^2}}\right)
\end{align}
where $\exp$ is the exponential function.
We set $\gamma$ to $0.01$ in this section.
Note that we use only the stocks that have the same target range
since $\mat{U}_{s_i} - \mat{U}_{s_j}$ is defined only when the two matrices are of the same size.

%Suppose that the target stock $s_T$ is \textit{Microsoft Corporation} (Ticker: MSFT), and the range is the COVID-19 duration (Jan. 2, 2020 - Apr. 15, 2021).
Based on $sim(s_i, s_j)$, we find similar stocks to $s_T$ using two different techniques: 1) $k$-nearest neighbors, and 2) Random Walks with Restarts (RWR).
The first approach simply finds stocks similar to the target stock, while the second one finds similar stocks by considering the multi-faceted relationship between stocks.

\textbf{$k$-nearest neighbors.}
We compute $sim(s_T, s_j)$ for $j=1,...,K$ where $K$ is the number of stocks to be compared, and find top-$10$ similar stocks to $s_T$, \textit{Microsoft} (Ticker: MSFT).
In Table~\ref{tab:similar_dist}, the \textit{Microsoft} stock is similar to stocks of the Technology sector or with a large capitalization (e.g., Amazon.com, Apple, and Alphabet) during the COVID-19.
Moody's is also similar to the target stock.

\begin{table*}[t]
\centering
\caption{
Based on the results of \method, we find similar stocks to Microsoft (MSFT) during the COVID-19.
(a) Top-$10$ stocks from $k$-nearest neighbors.
(b) Top-$10$ stocks from RWR.
The blue color refers to the stocks that appear only in one of the two approaches among the top-10 stocks.
}
\label{tab:similar_discovery}

\resizebox{0.4\textwidth}{!}{
\subtable[Similarity based Result]{
\label{tab:similar_dist}
\begin{tabular}{rrr}
\toprule
Rank & Stock Name & Sector 	\\
\midrule
1 & Adobe  & Technology      		\\
2 & Amazon.com  & Consumer Cyclical     		\\
3 & Apple  &    Technology 		\\
4 & Moody's  & Financial Services      		\\
5 & \textbf{\new{Intuit}} & Technology \\
6 & ANSYS & Technology      		\\
7 & Synopsys & Technology \\
8 & \textbf{\new{Alphabet}}  & Communication Services      		\\
9 & \textbf{\new{ServiceNow}}  & Technology      		\\
10 & \textbf{\new{EPAM Systems}}  & Technology      		\\
\bottomrule
\end{tabular}}}
\subtable[RWR Result]{
\label{tab:similar_ppr}
\resizebox{0.37\textwidth}{!}{
\begin{tabular}{rrr}
\toprule
Rank & Stock Name & Sector 	\\
\midrule
1 & Synopsys & Technology      		\\
2 & ANSYS & Technology      		\\
3 & Adobe & Technology      		\\
4 & Amazon.com	 & Consumer Cyclical      		\\
5 & \textbf{\new{Netflix}} & Communication Services      		\\
6 & \textbf{\new{Autodesk}} & Technology      		\\
7 & Apple & Technology      		\\
8 & Moody's & Financial Services      		\\
9 & \textbf{\new{NVIDIA}}	 & Technology      		\\
10 & \textbf{\new{S\&P Global}} & Financial Services		\\
\bottomrule
\end{tabular}}}
\end{table*}

\textbf{Random Walks with Restarts (RWR).}
We find similar stocks using another approach, Random Walks with Restarts (RWR)~\cite{10.1371/journal.pone.0265001,DBLP:journals/www/JungJPK21,DBLP:journals/kais/JungJK20,10.1371/journal.pone.0213857}.
To exploit RWR, we first a similarity graph based on the similarities between stocks.
The elements of the adjacency matrix $\mat{A}$ of the graph is defined as follows:
\begin{align}
		\mat{A}(i,j) = \begin{cases}
sim(s_i,s_j) &\text{if $i\neq j$}\\
0 &\text{if $i = j$}
\end{cases}
\end{align}
We ignore self-loops by setting $\mat{A}(i,i)$ to $0$ for $i=1,...,K$.

After constructing the graph, we find similar stocks using RWR.
The scores $\mat{r}$ is computed by using the power iteration~\cite{page1999pagerank} as described in~\cite{jung2017bepi}:
\begin{align}
	\mat{r}^{(i)} \leftarrow (1-c)\tilde{\mat{A}}^T\mat{r}^{(i-1)} + c\mat{q}
\end{align}
where $\tilde{\mat{A}}$ is the row-normalized adjacency matrix, $\mat{r}^{(i)}$ is the score vector at the $i$th iteration, $c$ is a restart probability, and $\mat{q}$ is a query vector.
We set $c$ to $0.15$, the maximum iteration to $100$, and $\mat{q}$ to the one-hot vector where the element corresponding to Microsoft is 1, and the others are 0.

%Table~\ref{tab:similar_ppr} shows top-$10$ rankings of RWR for \textit{Microsoft Corporation}.
As shown in Table~\ref{tab:similar_discovery}, the common pattern of the two approaches is that many stocks among the top-10 belong to the technology sector.
There is also a difference.
In Table~\ref{tab:similar_discovery}, the blue color indicates the stocks that appear only in one of the two approaches among the top-10.
In Table~\ref{tab:similar_dist}, the $k$-nearest neighbor approach simply finds the top-$10$ stocks which are closest to \textit{Microsoft} based on distances.
On the other hand, the RWR approach finds the top-$10$ stocks by considering more complicated relationships.
There are $4$ stocks appearing only in Table~\ref{tab:similar_ppr}.
S\&P Global is included since it is very close to Moody's which is ranked $4$th in Table~\ref{tab:similar_dist}.
Netflix, Autodesk,
and
NVIDIA are relatively far from the target stock compared to stocks such as Intuit and Alphabet, but they are included in the top-$10$ since they are very close to Amazon.com, Adobe, ANSYS, and Synopsys.
This difference comes from the fact that
the $k$-nearest neighbors approach considers only distances from the target stock while the RWR approach considers distances between other stocks in addition to the target stock.
%
%Table~\ref{tab:similar_discovery} describes the results for Apple and Samsung Electronics.
%The similarity patterns of the two stocks are different.
%In Table~\ref{tab:similar_apple}, Apple is similar to stocks with a large market capitalization in each sector such as Mastercard, Alphabet, Visa, and salesforce.com.
%As shown in Table~\ref{tab:similar_samsung}, the similarity pattern of Samsung Electronics is clearer than that of Apple.
%Almost all stocks are highly related to Samsung Electronics: 1) stocks naming \textit{Samsung} (ranks 1 and 3), 2) semiconductor-related stocks (ranks 2, 4, 5, 6, 8).
%Note that one of the main business areas of Samsung Electronics is semiconductors.

\method allows us to efficiently obtain factor matrices, and find interesting patterns in data. % various events in addition to the COVID-19.

\section{Related Works}
\label{sec:related}
We review related works on tensor decomposition methods for regular and irregular tensors.

\textbf{Tensor decomposition on regular dense tensors.}
There are efficient tensor decomposition methods on regular dense tensors.
Pioneering works~\cite{JangK20,MalikB18,WangTSA15,Tsourakakis10,BattaglinoBK18,GittensAY20} efficiently decompose a regular tensor by exploiting techniques that reduce time and space costs.
Also, a lot of works~\cite{PhanC11a,LiHCS16,AustinBK16,ChenHWZL17} proposed scalable tensor decomposition methods with parallelization to handle large-scale tensors.
However, the aforementioned methods fail to deal with the irregularity of dense tensors since they are designed for regular tensors.

\textbf{PARAFAC2 decomposition on irregular tensors.}
Cheng and Haardt~\cite{ChengH19} proposed \rdals which preprocesses a given tensor and performs PARAFAC2 decomposition using the preprocessed result.
However, \rdals requires high computational costs to preprocess a given tensor.
Also, \rdals is less efficient in updating factor matrices since it computes reconstruction errors for the convergence criterion at each iteration.
Recent works~\cite{PerrosPWVSTS17,AfsharPPSHS18,Ren00H20} attempted to analyze irregular sparse tensors.
SPARTan~\cite{PerrosPWVSTS17} is a scalable PARAFAC2-ALS method for large electronic health records (EHR) data.
COPA~\cite{AfsharPPSHS18} improves the performance of PARAFAC2 decomposition by applying various constraints (e.g., smoothness).
REPAIR~\cite{Ren00H20} strengthens the robustness of PARAFAC2 decomposition by applying low-rank regularization.
We do not compare \method with COPA and REPAIR since they concentrate on imposing practical constraints to handle irregular sparse tensors, especially EHR data. 
However, we do compare \method with \spartan which the efficiency of COPA and REPAIR is based on. 
%these methods propose effective regularization obtain factor matrices by adopting AO-ADMM which is a hybrid optimization framework based on alternating optimization (AO) and the alternating direction method of multipliers (ADMM).
TASTE~\cite{afshar2020taste} is a joint PARAFAC2 decomposition method for large temporal and static tensors.
Although the above methods are efficient in PARAFAC2 decomposition for irregular tensors, they concentrate only on irregular sparse tensors, especially EHR data.
LogPar~\cite{yin2020logpar}, a logistic PARAFAC2 decomposition method, analyzes temporal binary data represented as an irregular binary tensor.
SPADE~\cite{gujral2020spade} efficiently deals with irregular tensors in a streaming setting.
TedPar~\cite{yin2021tedpar} improves the performance of PARAFAC2 decomposition by explicitly modeling the temporal dependency.
Although the above methods effectively deal with irregular sparse tensors, especially EHR data,
none of them focus on devising an efficient PARAFAC2 decomposition method on irregular dense tensors.
On the other hand, \method is a fast and scalable PARAFAC2 decomposition method for irregular dense tensors.

%Helwig~\cite{helwig2017estimating} estimates trends for longitudinal data using PARAFAC2 decomposition. 

%\vspace{-2mm}
\section{Conclusion}
\label{sec:conclusion}
In this paper, we propose \method, a fast and scalable PARAFAC2 decomposition method for irregular dense tensors.
By compressing an irregular input tensor, careful reordering of the operations with the compressed results in each iteration, and careful partitioning of input slices,
\method successfully achieves high efficiency to perform PARAFAC2 decomposition for irregular dense tensors.
Experimental results show that \method is up to $6.0\times$ faster than existing PARAFAC2 decomposition methods while achieving comparable accuracy, and it is scalable with respect to the tensor size and target rank.
With \method, we discover interesting patterns in real-world irregular tensors.
Future work includes devising an efficient PARAFAC2 decomposition method in a streaming setting.

%\vspace{-3mm}
\section*{Acknowledgment}
%\vspace{-2mm}
This work was partly supported by the National Research Foundation of Korea(NRF) funded by 
MSIT(2022R1A2C3007921), 
and Institute of Information \& communications Technology Planning \& Evaluation(IITP) grant funded by MSIT [No.2021-0-01343, Artificial Intelligence Graduate School Program (Seoul National University)] 
and [NO.2021-0-02068, Artificial Intelligence Innovation Hub (Artificial Intelligence Institute, Seoul National University)].
The Institute of Engineering Research and ICT at Seoul National University provided research facilities for this work.
U Kang is the corresponding author.

\bibliographystyle{IEEEtran}
\bibliography{mybib}

\end{document}